\documentclass[letterpaper]{article} 
\usepackage{uai2020}  
\usepackage[margin=1in]{geometry}

\usepackage{times}
\usepackage{graphicx} 

\usepackage{natbib}
\setcitestyle{authoryear,open={(},close={)}}
\renewcommand{\refname}{References}
\makeatletter
\renewcommand{\bibsection}{%
	\subsubsection*{\refname%
		\@mkboth{\MakeUppercase{\refname}}{\MakeUppercase{\refname}}%
	}
}
\makeatother

\usepackage{amsmath}
\usepackage{amsfonts}
\usepackage{subcaption}
\usepackage{amsthm}
\usepackage{appendix}
\newtheorem{lemma}{Lemma}[section]
\newtheorem{theorem}[lemma]{Theorem}
\newtheorem{corollary}[lemma]{Corollary}

\DeclareMathOperator*{\argmin}{arg\,min}

\title{Robust Collective Classification against Structural Attacks}

\author{ {\bf Kai Zhou}\\
	Computer Science and Engineering\\
	Washington University in St. Louis\\
	Saint Louis, MO 63105\\
	zhoukai@wustl
	\And
	{\bf Yevgeniy Vorobeychik}\\
	Computer Science and Engineering\\
	Washington University in St. Louis\\
	Saint Louis, MO 63105 \\
	yvorobeychik@wustl.edu
}

\begin{document}

	\maketitle

	\begin{abstract}
		Collective learning methods exploit relations among
		data points to enhance classification
		performance. However, such relations, represented as
		edges in the underlying graphical model, expose an extra
		attack surface to the adversaries. We study adversarial
		robustness of an important class of such graphical
		models, Associative Markov Networks (AMN), to structural
		attacks, where an attacker can modify the graph
		structure at test time. We formulate the task of
		learning a robust AMN classifier as a bi-level
		program, where the inner problem is a challenging
		non-linear integer program that computes optimal
		structural changes to the AMN.  To address this
		technical challenge, we first relax the attacker problem,
		and then use duality to obtain a convex quadratic
		upper bound for the robust AMN problem.
		We then prove a bound on the quality of the resulting
		approximately optimal solutions, and experimentally
		demonstrate the efficacy of our approach.
		Finally, we apply our approach in a transductive
		learning setting, and show that robust AMN is much more robust
		than state-of-the-art deep learning methods, while
		sacrificing little in accuracy on non-adversarial data.

\end{abstract}
\section{INTRODUCTION}
Data from various domains can be compactly represented as graphs, such as social networks, citation networks, and protein interaction networks. In such graphical models, nodes, associated with attributes, represent the entities and edges indicate their relations. A common task is to classify nodes into different categories, e.g., classifying an account in a social network as malicious or benign.
Collective learning methods \citep{sen2008collective,taskar2002discriminative} exploit such relational structure in the data for classification tasks. 
For instance, in hypertext classification, the linked Web-pages tend to possess the same label, as they often share similar contents.  Graphical models \citep{koller2009probabilistic}, such as Markov networks \citep{richardson2006markov,taskar2004max,taskar2004learning,kok2005learning}, take into account such linking information in classifying relational data, exhibiting considerable performance improvement on a wide range of classification tasks \citep{li1994markov,taskar2004link,munoz2009contextual}.

However, making use of relational information in classification also exposes a new vulnerability in adversarial settings. Consider the task of classifying nodes in a social network as malicious or benign. 
If connectivity is used for this task, malicious parties have an incentive to manipulate information obtained about the network structure to reduce classification accuracy.
For example, malicious nodes may ostensibly cut ties among one another, or add links to benign nodes, with the purpose of remaining undetected.
While such structural attacks on collective learning have recently emerged~\citep{dai2018adversarial,zugner2018adversarial,DBLP:journals/corr/abs-1902-08412}, they have focused primarily on defeating neural network embedding-based approaches and transductive learning scenarios.

Our goal is to learn robust \emph{Associative Markov Networks} (AMN) \citep{taskar2004learning} that jointly classify the nodes in a given graph, where edges indicate that the connected nodes are likely to have the same label.
We formalize the problem of learning robust AMN as a bi-level program in which the inner optimization problem involves optimal structural attacks (adding or deleting edges).
The key technical challenge is that even the inner problem is a non-linear integer program, rendering the full optimization problem intractable.
We address this challenge by first relaxing the inner adversarial optimization, which allows us to approximate the full bi-level program by a convex quadratic upper bound that we can now efficiently minimize.
Our subsequent analysis first exhibits an approximation bound for the adversarial problem, and then an approximation bound on the solutions of our approximate robust AMN approach.

We test our approach on real datasets in several different settings. 

First, we show that structural attacks can effectively degrade the accuracy of AMN, with relational information now becoming an Achilles heel instead of an advantage. 
In contrast, robust AMN degrades gracefully, preserving the advantages of the use of relational information in classification, even with a relatively large adversarial modification of graph structure.
In addition, we compared robust AMN to a Graph Convolutional Network (GCN) classifier in a transductive learning setting under a structural attack, and show that robust AMN is significantly more robust than GCN, and is nearly as accurate as GCN on non-adversarial data.
This observation is particularly noteworthy given the fact that robust AMN was not specifically designed to be robust to transductive learning attacks, which effectively pollute both the training and test data.

\noindent{\bf Related Work} Our work falls into the realm of learning robust classifiers against decision-time reliability attacks~\citep{vorobeychik2018adversarial}.  There is a rich collection of prior work on decision-time attacks targeting a variety of classification approaches, ranging from classical models \citep{globerson2006nightmare,torkamani2013convex} to deep-learning methods \citep{eykholt2018robust,grosse2016adversarial}. As their countermeasures, several efforts \citep{li2014feature,li2018evasion,goodfellow2014explaining,madry2018towards} have been devoted to enhancing the robustness of classifiers in adversarial settings.
A fundamental difference in our work is that  we are defending against structural attacks that exploit the relations among data points for the purpose of attacking, while most prior work on robust learning considers settings that treat data independently.

More closely related to our focus are several prior studies of the vulnerability and robustness of collective learning models. \citeauthor{torkamani2013convex}~(\citeyear{torkamani2013convex}) also considered the robustness of AMN to prediction-time attacks, but consider attacks that modify node features, leaving robustness to structural attacks as an open problem.
Recently there have been a number of demonstrations of vulnerabilities of graph neural networks to attacks in semi-supervised (transductive learning) settings~\citep{dai2018adversarial,zugner2018adversarial,DBLP:journals/corr/abs-1902-08412}.
Specifically, \cite{dai2018adversarial} and \cite{zugner2018adversarial} proposed targeted attacks aiming at misclassifying a subset of target nodes while \cite{DBLP:journals/corr/abs-1902-08412} focused on reliability attacks at training time.
While our focus is not on such transductive learning problems (in which attacks also poison the training data), we explore the robustness of our approach in such settings in the experiments below. Other than collective learning, there are a host of works studying structural attacks as well as defense approaches in a more general setting of network analysis tasks, such as link prediction~\citep{zhou2019attacking,zhou2019adversarial}, community detection~\citep{waniek2018hiding} and so on.

\section{BACKGROUND}
A Markov network is defined over an undirected graph $\mathcal{G} = (V,E)$, where a node $v_i \in V$, $i = 1,2,\cdots,N$, is associated with a \emph{node variable} $Y_i$, representing an object to be classified. We assume there are $K$ discrete labels, i.e., $Y_i \in \{1,2,\cdots,K\}$. At a high level, a Markov network defines a joint distribution of $Y = \{Y_1,Y_2,\cdots,Y_N\}$, which is a normalized product of \emph{potentials}: $P_{\phi} (\mathbf{y}) =\frac{1}{Z}\prod_{c\in \mathcal{C}} \phi_c (\mathbf{y}_c)$, where $\phi_c({\mathbf{y}_c})$ is a potential function associated with each \emph{clique} $c \in \mathcal{C}$ in $\mathcal{G}$, and $Z$ is a normalization factor. The potential function $\phi_c(\mathbf{y}_c)$ maps a \emph{label assignment} $\mathbf{y}_c$ of the nodes in the clique $c$ to a non-negative value.

Our focus is the \emph{pairwise Associative} Markov Networks (AMN) \citep{taskar2004learning}, where each clique is either a node or a pair of nodes. Thus, the joint distribution of $Y$ can be written as
\begin{eqnarray}
	P_{\phi} (\mathbf{y}|\mathcal{G}) = \frac{1}{Z} \prod_{i=1}^N \phi_i(y_i) \prod_{(i,j)\in E} \phi_{ij}(y_i,y_j),
\end{eqnarray}
where we make explicit the dependency of the probability on the network structure $\mathcal{G}$. In detail, let $y_i^k$ be a binary indicator where $y_i^k = 1$ means that node $i$ is assigned label $k$. The node and edge potentials then join such label assignments with the node and edge features, respectively. Specifically, let $\mathbf{x}_i \in \mathbb{R}^{d_n}$ and $\mathbf{x}_{ij} \in \mathbb{R}^{d_e}$ be the feature vectors of node $i$ and edge $(i,j)$. In the log-linear model, the potential functions are defined as $\log \phi_i (y_i^k) = \mathbf{w}_n^k \cdot \mathbf{x}_i$ and $\log \phi_{ij}(y_i^k, y_j^{k'}) = \mathbf{w_e}^{k,k'} \cdot \mathbf{x}_{ij}$, where $\mathbf{w}_n^k \in \mathbb{R}^{d_n}$ and $\mathbf{w}_e^{k,k'} \in \mathbb{R}^{d_e}$ are node and edge parameters. Note that such parameters are \emph{label-specific} in that they are different with respect to labels $k, k'$, and the same for the nodes and edges, respectively.

In \emph{associative} MN, the link $(i,j)$ indicates an associative relation between nodes $i$ and $j$, meaning that $i$ and $j$ tend to be classified as the same label. Reflected on the edge potentials, it is assumed that $\phi_{ij}(y_i^k,y_j^{k'}) = 1$ for any different labels $k, k' \in \{1,2,\cdots,K\}$ while $\phi_{ij}(y_i^k,y_j^k)$ equals some value greater than $1$. Consequently, the edge potentials associated with those edges connecting differently labeled nodes are $0$ in the log space. For simplicity, we write the edge parameters for label $k$ as $\mathbf{w}_e^k$.  Putting everything together, the AMN defines the log of conditional probability as 
\begin{align}
	\label{eqn-log-prob}
	&\log P_\mathbf{w}(\mathbf{y}|\mathbf{x},\mathcal{G})\\
	=& \sum_{i=1}^N \sum_{k=1}^K (\mathbf{w}_n^k \cdot \mathbf{x}_i) y_i^k  + \sum_{(i,j)\in E} \sum_{k=1}^K (\mathbf{w}_e^{k} \cdot \mathbf{x}_{ij}) y_i^ky_j^k \nonumber\\
	&- \log Z_\mathbf{w}(\mathbf{x}),\nonumber
\end{align}
where $\mathbf{w}$ and $\mathbf{x}$ represent all the node and edge parameters and features, and $\mathbf{y}$ represents a label assignment. Importantly, the last term $Z_\mathbf{w} (\mathbf{x})$ does not depend on the assignment $\mathbf{y}$.

Two essential tasks of AMN are \emph{inference} and \emph{learning}. In inference, one seeks the optimal assignment $\mathbf{y}$ that maximizes the log conditional probability $\log P_{\mathbf{w}}(\mathbf{y}|\mathbf{x},\mathcal{G})$ (excluding the term $Z_\mathbf{w}(\mathbf{x})$), given observed features $\mathbf{x}$ and learned parameters $\mathbf{w}$. \cite{taskar2004learning}  showed that in AMN, such an inference problem can be (approximately) solved in polynomial time on arbitrary graph structures (when $K =2$, the solution is optimal). To learn the weights $\mathbf{w}$, \cite{taskar2004max} proposed a maximum margin approach by maximizing the gap between the confidence in the true labeling $\hat{\mathbf{y}}$ and any alternative labeling $\mathbf{y}$, denoted by  $ \Delta P_\mathbf{w}(\hat{\mathbf{y}}, \mathbf{y}|\mathbf{x},\mathcal{G}) = \log P_\mathbf{w}(\hat{\mathbf{y}}|\mathbf{x},\mathcal{G}) - \log P_\mathbf{w}(\mathbf{y}|\mathbf{x},\mathcal{G})$. Specifically, they formulate the learning problem as follows:
\begin{align}
	\label{eqn-qp}
	\min\quad  &\frac{1}{2}||\mathbf{w}||^2 + C \xi,\\
	\text{s.t.}\quad &\xi \geq \max_{\mathbf{y} \in \mathcal{Y}} (\Delta (\hat{\mathbf{y}},\mathbf{y}) - \Delta P_\mathbf{w}(\hat{\mathbf{y}}, \mathbf{y}|\mathbf{x},\mathcal{G})).\nonumber
\end{align}
By relaxing the integrality constraints and using strong duality of linear programming, the inner maximization problem is replaced with its dual minimization problem. As a result, the weights $\mathbf{w}$ can be efficiently learned through solving the quadratic program with linear constraints.	

\section{LEARNING ROBUST AMN}

\subsection{MODEL}

In max-margin AMN learning, the exponential set of constraints is replaced by a single most-violated constraint.
In the adversarial setting, the attacker is also modifying the structure of $\mathcal{G}$, potentially strengthening the constraint in Eqn.~\eqref{eqn-qp}. Our robust formulation extends the max-margin learning formulation, taking into account the change caused by modifications of $\mathcal{G}$.

We begin by considering an attacker who can delete existing edges from $\mathcal{G}$, and subsequently show that our model can be extended to the case where an attacker can both add and delete edges. To formalize, we assign a binary decision variable $e_{ij}$  for each edge $(i,j) \in E$, where $e_{ij} = 0$ means that the attacker decides to delete that edge. Then $\mathbf{e} = (e_{ij})_{(i,j)\in E}$ is the decision vector of the attacker.
Let $\mathcal{E} = \{\mathbf{e}: e_{ij} \in \{0,1\}; \sum_{(i,j)\in E} e_{ij} \geq |E| - D^-\}$ be the space of all the possible decision vectors, where $D^-$ is a budget on the number of edges that the attacker can delete. Then robust learning can be formulated as
\begin{align}
	\label{eqn-robust}
	\min\quad  &\frac{1}{2}||\mathbf{w}||^2 + C \xi,\\
	\text{s.t.}\quad &\xi \geq  \max_{\mathbf{y} \in \mathcal{Y}, \mathbf{e}\in \mathcal{E}} (\Delta (\hat{\mathbf{y}},\mathbf{y}) - \Delta P_\mathbf{w}(\hat{\mathbf{y}}, \mathbf{y}|\mathbf{x},\mathcal{G}(\mathbf{e}))),\nonumber
\end{align}
where $\mathcal{G}(\mathbf{e})$ denotes the modified graph obtained by deleting edges as indicated in $\mathbf{e}$ from $\mathcal{G}$. 
Henceforth, we omit the edge features $\mathbf{x}_{ij}$ and write $\mathbf{w}_e^k \mathbf{x}_{ij}$ as $w_e^k$ to simplify notation.
From Eqn.~\eqref{eqn-log-prob}, we can rewrite Eqn.~\eqref{eqn-robust} as
\begin{subequations}
	\label{eqn-robust-detail}
	\begin{align}
		\min\quad  &\frac{1}{2}||\mathbf{w}||^2 + C \xi,\\
		\text{s.t.}\quad &\xi \geq \max_{\mathbf{y} \in \mathcal{Y}, \mathbf{e}\in \mathcal{E}} 
		\sum_{i=1}^N \sum_{k=1}^K (\mathbf{w}_n^k \mathbf{x}_i)(y_i^k - \hat{y}_i^k) \label{C:attack}\\
		&+\sum_{(i,j) \in E} \sum_{k=1}^K w_e^k(y_i^k y_j^k - \hat{y}_i^k \hat{y}_j^k)\cdot e_{ij} \nonumber\\
		& + N - \sum_{i=1}^N\sum_{k=1}^K \hat{y}_i^k y_i^k.\nonumber
	\end{align}
\end{subequations}
In this formulation, the attacker's choice of which edges to remove is captured by the inner maximization problem over $\mathbf{e}$ inside Constraint~\eqref{C:attack}.
We call this attack, where edges can only be deleted, \textit{Struct-D}.

A natural extension is to consider a strong attacker who can simultaneously delete and add edges. We term such an attack \textit{Struct-AD} (\emph{AD} for adding and deleting edges).
A critical difference between \textit{Struct-D} and \textit{Struct-AD} is that the search space of which ``non-edges'' should be added is significantly larger than that of edges to be removed, since graphs tend to be sparse.
This dramatically increases the complexity of solving Eqn.~\eqref{eqn-robust-detail} (in particular, of enforcing the constraint associated with computing the optimal attack). 
To address this, we restrict the attacker to only add edges between two data points that have different labels, resulting in a reduced set of non-edges denoted by $\bar{E}$. 
The intuition behind this restriction is that adding edges between nodes with the same label provides useful information for classification, and is unlikely to be a part of an optimal attack; rather, the attacker would focus on adding edges between pairs of nodes with different labels to increase classification noise.

For the \textit{Struct-AD} attack, we use a binary decision variable $\bar{e}_{ij}$ for each non-edge $(i,j) \in \bar{E}$, where $\bar{e}_{ij} = 1$ means that the attacker chooses to add an edge between $i$ and $j$. As a by-product, each term $w_e^k \hat{y}_i^k \hat{y}_j^k \bar{e}_{ij}$ becomes $0$. Then in the formulation of \textit{Struct-AD}, we only need to add terms  $\sum_{(i,j)\in \bar{E}}\sum_{k=1}^K w_e^k y_i^k y_j^k \bar{e}_{ij}$ in the attacker's objective and an extra linear constraint $\sum_{(i,j) \in \bar{E}} \bar{e}_{ij} \leq D^+$, where $D^+$ is the constraint on the number of edges that the attacker can add.
We can further extend the formulation to allow additional restrictions on the attacker, such as limiting the change to node degrees; indeed, we can accommodate the addition of any linear constraints on the attack.


The weights of the robust AMN are learned through solving the bi-level optimization problem above.
However, this is a challenging task: first, even the inner maximization problem (optimal structural attack) is a combinatorial optimization problem, and the bi-level nature of the underlying formulation makes it all the more intractable.
Next, we present an efficient approximate solution to robust AMN learning with provable guarantees both for the attack subproblem, and to the overall robust learning problem.
We focus on formulation~\eqref{eqn-robust-detail} to simplify exposition, but all our results generalize to the setting with \textit{Struct-AD} attacks.

\subsection{APPROXIMATE SOLUTION}
Our solution is based on approximating the inner-layer non-linear integer program by a linear program (LP). To this end, we first linearize the non-linear terms (product of three binary variables) using standard techniques, and then relax the integrality constraints. 
Finally, we use LP duality to obtain a single convex quadratic program for robust AMN learning, thereby minimizing an upper bound on the original bi-level program in Eqn.~\eqref{eqn-robust-detail}.
Subsequently, we provide approximation guarantees for the resulting solutions.

To begin, we replace each non-linear term $y_i^k y_j^k e_{ij}$ in Eqn.~\eqref{eqn-robust-detail} with a non-negative continuous variable $z_{ij}^k$ and add three linear constraints $z_{ij}^k \leq y_i^k$, $z_{ij}^k \leq y_j^k$, and $z_{ij}^k \leq e_{ij}$. We omit the constraint $z_{ij}^k \geq y_i^k + y_j^k + e_{ij} -2$, since we are maximizing the objective (in the inner optimization problem) and the weights $w_e^k$ are non-negative; consequently, the optimal solution takes the value $z_{ij}^k = \min \{y_i^k,y_j^k,e_{ij}\}$, which is equivalent to $y_i^k y_j^k e_{ij}$ in the binary case. 
We further relax the integrality constraints on the binary variables $\mathbf{y}$ and $\mathbf{e}$, resulting in a linear program to approximate the attacker's problem, omitting the constant terms in $\mathbf{y}$ and $\mathbf{e}$:
\begin{align}
	\label{eqn-atk-lp}
	\max_{\mathbf{y},\mathbf{e}}\quad &\sum_{i=1}^N \sum_{k=1}^K (\mathbf{w}_n^k \mathbf{x}_i - \hat{y}_i^k)y_i^k 
	+ \sum_{(i,j)\in E}\sum_{k=1}^K w_e^k  z_{ij}^k\\
	&- \sum_{(i,j)\in E}\sum_{k=1}^K w_e^k  \hat{y}_i^k \hat{y}_j^k e_{ij} \nonumber \\
	\text{s.t.}\quad  &\forall i, \sum_{k=1}^K y_i^k = 1, y_i^k \geq 0; 
	\forall (i,j) \in E, e_{ij} \in [0,1]; \nonumber \\
	&|E| - \sum_{(i,j) \in E} e_{ij} \leq D^-; \nonumber \\
	&\forall (i,j) \in E, \forall k, z_{ij}^k \leq y_i^k, z_{ij}^k \leq y_j^k, z_{ij}^k \leq e_{ij}\nonumber
\end{align}

By LP duality, we can replace the attacker's maximization problem using its dual minimization problem, which is further integrated into Eqn.~\eqref{eqn-robust-detail}. Consequently, we can approximate Eqn.~\eqref{eqn-robust-detail} by a convex quadratic program (QP), which is presented in the appendix.
We can use the same techniques to formulate a corresponding quadratic program for \textit{Struct-AD}.

When the LP relaxation defined in Eqn.~\eqref{eqn-atk-lp} produces integral solutions of $\mathbf{y}$ and $\mathbf{e}$, the QP will produce optimal weights. In the case where the LP's solutions are fractional, we obtain an upper bound on the true objective function.
Thus, the approximation quality of LP determines the gap between the true and approximate objective values for the defender and, consequently, the gap between approximate and optimal solutions to robust AMN. In Section~\ref{sec-bound} we bound this gap.

\section{BOUND ANALYSIS}
\label{sec-bound}
The key to analyzing the bound of the defender's objective is to devise a well-approximated \textit{integral} solution to the attacker's problem as defined in Eqn.~\eqref{eqn-robust-detail}. On the one hand, such an integral solution bridges the gap between the optimal integral solution of the attacker's problem and the optimal fractional solution of its LP relaxation. On the other hand, it generates effective structural attacks to test our robust AMN model. We first focus on approximating structural attacks and then analyze how to transfer the bound on the attack performance to the bound on the defender's objective.

\subsection{APPROXIMATING STRUCTURAL ATTACKS}
Given fixed weights of the AMN model, the attacker solves the LP in~\eqref{eqn-atk-lp} and determine which edges to delete. Unfortunately, Eqn.~\eqref{eqn-atk-lp} will produce fractional solutions, meaning that the attacker needs to round the results to obtain a feasible (but not necessarily optimal) attack. 
\cite{kleinberg2002approximation} proposed a randomized rounding scheme to assign labels to nodes in a class of classification problems with pairwise relationships. We follow the idea and apply it to our case where we are simultaneously assigning $K$ labels to the nodes and two labels (delete or not) to the edges. 

Given the optimal solution $(\mathbf{y}^*, \mathbf{e}^*,\mathbf{z}^*)$ of LP \eqref{eqn-atk-lp},  the randomized rounding procedure (termed \textit{RRound}) produces a corresponding integral solution $(\mathbf{y}^I, \mathbf{e}^I)$. Specifically, \textit{RRound} rounds $\mathbf{y}^*$ and $\mathbf{e}^*$ in \emph{phases}. In a single phase, we independently draw a label $k \in \{1,2,\cdots,K\}$ and a label $b \in \{0,1\}$ (where $b = 0$ means the corresponding edge is chosen to be deleted) at the beginning. We assign this specific label $k$ to each node and $b$ to each edge in a probabilistic way. Specifically, we generate a continuous random number $\beta$ uniformly from $[0,1]$. For each node $i$, if $\beta \leq y_i^{k*}$, we assign the label $k$ to $i$ (i.e., set $y_i^k = 1$). For each edge $(i,j) \in E$, if $\beta \leq e_{ij}^{b*}$, where $e_{ij}^{1*} = e_{ij}^*$ and $e_{ij}^{0*} = 1 - e_{ij}^{*}$, we assign the label $b$ to $(i,j)$, i.e., $e_{ij} = b$. \textit{RRound} stops when all nodes and edges are assigned labels. For the auxiliary variable $z_{ij}^k$, it takes the minimum of the rounded $y_i^k$, $y_j^k$, and $e_{ij}$, which is equivalent to $y_i^k \cdot y_j^k \cdot e_{ij}$ in the binary space. We thus omit $\mathbf{z}^I$ and specify the output of \textit{RRound} as $(\mathbf{y}^I, \mathbf{e}^I)$.

To simplify presentation, we write the attacker's \textit{approximate} objective in Eqn.~\eqref{eqn-atk-lp} as $A(\mathbf{y},\mathbf{e},\mathbf{z};\mathbf{w}) = A_1(\mathbf{y},\mathbf{e};\mathbf{w}) + A_2(\mathbf{z};\mathbf{w})$, where $A_1(\mathbf{y},\mathbf{e};\mathbf{w}) = \sum_{i=1}^N \sum_{k=1}^K (\mathbf{w}_n^k \mathbf{x}_i - \hat{y}_i^k)y_i^k - \sum_{(i,j)\in E}\sum_{k=1}^K w_e^k  \hat{y}_i^k \hat{y}_j^k e_{ij}$ and $A_2(\mathbf{z};\mathbf{w}) = \sum_{(i,j)\in E}\sum_{k=1}^K w_e^k  z_{ij}^k$.  Similarity, the attacker's \textit{true} objective specified in \eqref{eqn-robust-detail} (inner problem) is denoted as $A^I(\mathbf{y},\mathbf{e}) = A_1^I(\mathbf{y},\mathbf{e};\mathbf{w}) + A_2^I(\mathbf{y},\mathbf{e};\mathbf{w})$, where $A_1^I(\mathbf{y},\mathbf{e};\mathbf{w}) = \sum_{i=1}^N \sum_{k=1}^K (\mathbf{w}_n^k \mathbf{x}_i - \hat{y}_i^k)y_i^k - \sum_{(i,j)\in E}\sum_{k=1}^K w_e^k  \hat{y}_i^k \hat{y}_j^k e_{ij}$ and $A_2(\mathbf{z};\mathbf{w}) = \sum_{(i,j)\in E}\sum_{k=1}^K w_e^k  y_i^k y_j^k e_{ij}^k $, with the constraint that $\mathbf{y}$ and $\mathbf{e}$ take binary values.

Our primary interest is then to  analyze the gap between $A^I(\mathbf{y}^I, \mathbf{e}^I;\mathbf{w})$ and $A(\mathbf{y}^*,\mathbf{e}^*,\mathbf{z}^*;\mathbf{w})$. The following lemma is a direct extension of Lemma B.1 of \citep{taskar2004learning}.
\begin{lemma}[\citep{taskar2004learning}]
	\label{lemma-1}
	RRound assigns label $k$ to node $i$ with probability $y_i^{k*}$ and assigns label $1$ to edge $(i,j)$ with probability $e_{ij}^{*}$.
\end{lemma}

Next, Lemma~\ref{lemma-2} gives a bound on the probability that edge $(i,j)$ is not deleted and node $i$ and $j$ are assigned the same label.
\begin{lemma}
	\label{lemma-2}
	RRound assigns label $k$ to node $i$ and node $j$ and assigns label $1$ to edge $(i,j)$ simultaneously with probability at least $\frac{1}{K + 4} z_{ij}^{k*}$.
\end{lemma}
\begin{proof}
	We consider the assignment in \emph{a single phase} when nodes $i$, $j$ and edge $(i,j)$ are not assigned any labels at the beginning of this phase. Note that in every phase, there are a total $2K$ combinations of the random draws $(k,b)$ with equal probabilities. Then the probability that $i$ and $j$ are assigned $k$ and $(i,j)$ is assigned $1$ is $\frac{1}{2K} \min \{y_i^{k*},y_j^{k*},e_{ij}^{1*}\} = \frac{1}{2K} z_{ij}^{k*}$, since the optimal solution of LP satisfies $z_{ij}^{k*} = \min \{y_i^{k*},y_j^{k*},e_{ij}^{1*}\}$. Consider a specific draw $(k,b)$, the probability that at least one of $i$ and $j$ is assigned label $k$ \emph{or} edge $(i,j)$ is assigned a label is $\frac{1}{2K}\max \{y_i^{k*},y_j^{k*},e_{ij}^{b*}\}$. By summing over all combination of $(k,b)$, we have the probability that none of $i$, $j$, and $(i,j)$ is assigned \emph{any} label is $1-\sum_{(k,b)} \frac{1}{2K}\max \{y_i^{k*},y_j^{k*},e_{ij}^{b*}\} = 1- \frac{1}{2K} \sum_{k=1}^K ( \max \{y_i^{k*},y_j^{k*},e_{ij}^{0*}\} + \max \{y_i^{k*},y_j^{k*},e_{ij}^{1*}\})$.
	
	Now consider all the phases. Nodes $i$ and $j$ and edge $(i,j)$ can be assigned label $k$ and $1$ simultaneously in a single phase or in separate phases. For the purpose of deriving the lower-bound, we only consider the probability that $i$ and $j$ are assigned $k$ and $(i,j)$ is assigned $1$ in a single phase. Summing over all phases, the probability that $i$ and $j$ are assigned $k$ and $(i,j)$ is assigned $1$ is at least
	\begin{align*}
		&\sum_{p=1}^{\infty}\frac{1}{2K} z_{ij}^{k*}[1-\frac{1}{2K} \sum_{k=1}^K ( \max \{y_i^{k*},y_j^{k*},e_{ij}^{0*}\} \\
		&+ \max \{y_i^{k*},y_j^{k*},e_{ij}^{1*}\})]^{p-1}\\
		=&  \frac{z_{ij}^{k*}}{\sum_{k=1}^K ( \max \{y_i^{k*},y_j^{k*},e_{ij}^{0*}\} + \max \{y_i^{k*},y_j^{k*},e_{ij}^{1*}\}}\\
		\geq &\frac{z_{ij}^{k*}}{\sum_{k=1}^K (y_i^{k*}+y_j^{k*}+e_{ij}^{0*} + y_i^{k*}+y_j^{k*}+e_{ij}^{1*})} \\
		=& \frac{z_{ij}^{k*}}{K + 4}.
	\end{align*}
	The first equality comes from $\sum_{i = 0}^{\infty} d^i = \frac{1}{1-d}$ for $d<1$.
\end{proof}

Finally, we can use both of these lemmas
to prove a lower bound on the expected value of $A^I(\mathbf{y}^I, \mathbf{e}^I;\mathbf{w})$:
\begin{theorem}
	\label{thm-atk-bound}
	Let $\mathbb{E}[A^I(\mathbf{y}^I, \mathbf{e}^I;\mathbf{w})]$ be the expected value of $A^I(\mathbf{y}^I, \mathbf{e}^I;\mathbf{w})$. Then $\mathbb{E}[A^I(\mathbf{y}^I, \mathbf{e}^I;\mathbf{w})] \geq A_1(\mathbf{y}^*,\mathbf{e}^*;\mathbf{w}) + \frac{1}{K+4} A_2(\mathbf{z}^*;\mathbf{w})$
\end{theorem}
\begin{proof}
	From Lemma~\ref{lemma-1} and Lemma~\ref{lemma-2}, we have $\mathbb{E}[A_1^I(\mathbf{y}^I, \mathbf{e}^I; \mathbf{w})] = A_1(\mathbf{y}^*,\mathbf{e}^*;\mathbf{w})$ and $\mathbb{E}[A_2^I(\mathbf{y}^I, \mathbf{e}^I; \mathbf{w})] \geq \frac{1}{K+4} A_2(\mathbf{z}^*;\mathbf{w})$, giving the claimed result.
\end{proof}

\paragraph{Derandomization} We propose a \emph{semi-derandomized} algorithm (termed \textit{Semi-RRound}) to select the edges. The key observation is that once $\mathbf{y}$ are determined, the objective is a linear function with respect to $\mathbf{e}$. Thus, by selecting those edges in ascending order of their coefficients, the objective is maximized.  Let $\mathbf{e}^{I'}$ be the decision vector for the edges output by \textit{Semi-RRound}.  Specifically, \text{Semi-RRound} takes an output $(\mathbf{y}^I,\mathbf{e}^I)$ of \textit{RRound} and compute the objective as a linear function of $\mathbf{e}$. It then deletes edges $(i,j)$ (sets $e_{ij}=0$) in ascending order of the coefficients of $e_{ij}$ until $D^-$ edges are deleted. Let $\mathbf{e}^{I'}$ be the decision vector for the edges output by \textit{Semi-RRound} and $A^I(\mathbf{y}^I,\mathbf{e}^{I'};\mathbf{w})$ be the corresponding objective. The following Corollary shows that the expected value of $A^I(\mathbf{y}^I,\mathbf{e}^{I'};\mathbf{w})$ retains the lower bound in Theorem~\ref{thm-atk-bound}.

\begin{corollary}
	\label{coro-round}
	$\mathbb{E}[A^I(\mathbf{y}^I, \mathbf{e}^{I'};\mathbf{w})] \geq A_1(\mathbf{y}^*,\mathbf{e}^*;\mathbf{w}) + \frac{1}{K+4} A_2(\mathbf{z}^*;\mathbf{w})$.
\end{corollary}

\begin{proof}
	For each $(\mathbf{y}^I, \mathbf{e}^I)$, \textit{Semi-RRound} finds the optimal $\mathbf{e}^{I'}$ that maximizes $A^I(\mathbf{y}^I,\mathbf{e};\mathbf{w})$. Thus, $A^I(\mathbf{y}^I, \mathbf{e}^{I'};\mathbf{w}) \geq A^I(\mathbf{y}^I, \mathbf{e}^{I};\mathbf{w})$ for every rounded output $(\mathbf{y}^I, \mathbf{e}^I)$. By the definition of expectation, we have $\mathbb{E}[A^I(\mathbf{y}^I, \mathbf{e}^{I'};\mathbf{w})] \geq A_1(\mathbf{y}^*,\mathbf{e}^*;\mathbf{w}) + \frac{1}{K+4} A_2(\mathbf{z}^*;\mathbf{w})$.
\end{proof}

\subsection{BOUNDING THE DEFENDER'S OBJECTIVE}
We rewrite the defender's objective in Eqn.~\eqref{eqn-robust-detail} as 
\begin{equation}
	\label{eqn-def-obj}
	\min \ F_D(\mathbf{w}) = D(\mathbf{w}) + C \cdot \max_{\mathbf{y} \in \mathcal{Y}, \mathbf{e}\in \mathcal{E}} A^I(\mathbf{y},\mathbf{e};\mathbf{w}),
\end{equation}
where $D(\mathbf{w}) = \frac{1}{2}||\mathbf{w}||^2 + C(N- \sum_{i=1}^N \sum_{k=1}^K \mathbf{w}_n^k \mathbf{x}_i \hat{y}_i^k)$. 
Note that the inner maximization problem implicitly defines a function of $\mathbf{w}$, which we denote by $F_A(\mathbf{w}) = A^I(\mathbf{y}^{I*},\mathbf{e}^{I*};\mathbf{w})$. In solving Eqn.~\eqref{eqn-def-obj}, we are approximating $A^I(\mathbf{y},\mathbf{e};\mathbf{w})$ using its LP relaxation; that is, the approximated inner-layer maximization defines a new function $\hat{F}_A(\mathbf{w}) = A(\mathbf{y}^*,\mathbf{e}^*,\mathbf{z}^*;\mathbf{w})$. Thus, instead of directly minimizing the \textit{actual} objective $F_D(\mathbf{w}) = D(\mathbf{w}) + C F_A(\mathbf{w})$, the defender is minimizing an \textit{approximated} (upper-bound) objective $\hat{F}_D(\mathbf{w}) = D(\mathbf{w}) + C \hat{F}_A(\mathbf{w})$.

Let $\mathbf{w}^* = \argmin F_D(\mathbf{w})$ and $\mathbf{w}^\diamond = \argmin \hat{F}_D(\mathbf{w})$. The following theorem bounds the difference between $\hat{F}_D(\mathbf{w}^\diamond)$ and $F_D(\mathbf{w}^*)$.

\begin{theorem}
	Let $\epsilon = \max \{  w_e^{*1},w_e^{*2},\cdots,w_e^{*K} \}$. Then, $\hat{F}_D(\mathbf{w}^\diamond) - F_D(\mathbf{w}^*) \leq \frac{(K+3)C|E|}{K+4}\epsilon$.
\end{theorem}

\begin{proof}
	Since $\mathbf{w}^\diamond$ minimizes $\hat{F}_D(\mathbf{w})$, we have $\hat{F}_D(\mathbf{w}^\diamond) \leq \hat{F}_D(\mathbf{w}^*)$. Then
	\begin{align}
		&\hat{F}_D(\mathbf{w}^\diamond) - F_D(\mathbf{w}^*)\\
		\leq &\  \hat{F}_D(\mathbf{w}^*) -  F_D(\mathbf{w}^*)\nonumber \\
		= &\  C\cdot \hat{F}_A(\mathbf{w}^*) - C\cdot F_A(\mathbf{w}^*)\nonumber \\
		= &\   C(A(\mathbf{y}^*,\mathbf{e}^*,\mathbf{z}^*;\mathbf{w}^*) - A^I(\mathbf{y}^{I*},\mathbf{e}^{I*};\mathbf{w}^*)) \nonumber \\
		\leq &\   C\frac{K+3}{K+4}A_2(\mathbf{z}^*;\mathbf{w}^*).
	\end{align}
	The second inequality is from the result of Corollary~\ref{coro-round} and the fact that  $(\mathbf{y}^{I*},\mathbf{e}^{I*})$ is the integral optimal solution. Then $A_2(\mathbf{z}^*;\mathbf{w}^*) = \sum_{(i,j)\in E}\sum_{k=1}^K w_e^{*k} z_{ij}^{*k}$. As $z_{ij}^{*k} \leq y_i^{*k}$ and $\sum_{k=1}^K y_i^{*k} = 1$ in the optimal fractional solution, we have $A_2(\mathbf{z}^*;\mathbf{w}^*) \leq \sum_{(i,j)\in E} \epsilon \sum_{k=1}^K y_i^{*k} = |E|\epsilon$. Thus $\hat{F}_D(\mathbf{w}^\diamond) - F_D(\mathbf{w}^*) \leq \frac{(K+3)C|E|}{K+4}\epsilon$.
\end{proof}

We note that the bound analysis can be extended to the case where the attacker can delete and add links simultaneously. One difference is that instead of using the inequality $z_{ij}^* \leq y_i^{*k}$, we can use $\bar{z}_{ij}^{*k} \leq \bar{e}_{ij}^{*}$, where $\bar{z}_{ij}^{k}$ is the auxiliary variable to represent the product $y_i^k y_j^k \bar{e}_{ij}$. This leads to a bound $\frac{(K+3)C(|E|+ K\cdot D^+)}{K+4}\epsilon$, where $D^+$ is the budget on the number of added edges.
\section{EXPERIMENTS}
We test the performance of the robust AMN classifier (henceforth R-AMN) under structural attacks as well as a recently proposed attack on Graph Convolutional Networks (GCN).
We focus on the binary classification task as AMN will in this case learn the optimal weights.

\subsection{DATASETS} We consider four real-world datasets:
Reuters, WebKB, Cora and CiteSeer. For the Reuters dataset
\citep{taskar2004learning,torkamani2013convex}, we follow the procedure
in \citep{taskar2004learning} to
extract four categories (``trade'', ``crude'', ``grain'', and
``money-fx'') of documents where each document belongs to a unique
category. We create a binary dataset by fixing ``trade'' as the
positive class and randomly drawing an equal number of documents from
the other three categories. We extract the bag-of-words representation
for each document. To construct the links, we connect each document to
its three closest neighbors in terms of the cosine distance of TF-IDF
representations, resulting in a graph of $862$ nodes and $1860$
edges. The WebKB dataset \citep{craven1998learning} contains webpages
(represented as binary vectors of dimension $1703$) from four
universities, which are classified into five classes. We fix
``student'' as the positive class and the rest as the negative class. We connect each webpage to its $3$ closest neighbors based on cosine distance, resulting in a network of $877$ nodes and $2282$ edges. Cora \citep{mccallum2000automating} and CiteSeer \citep{giles1998citeseer} are two citation networks, where the nodes in the graph are the bag-of-words representations of papers and edges are the citation relations among them. For Cora, we use the ``Prob-methods'' as the positive class and randomly draw an equal number of papers from the other categories to form the negative class, resulting in a subgraph of $852$ nodes and $838$ edges. Similarly for CiteSeer, we fix ``agents'' as the positive class, resulting in a subgraph of $1192$ nodes and $953$ edges. Cora and CiteSeer have more sparse graph structures than the Reuters and WebKB. Moreover, Reuters and WebKB tend to be  regular graphs while a few nodes in Cora and CiteSeer have large degrees. We split each dataset evenly into training and test data for our experiments.

\begin{figure}[t!]
	\centering
	\begin{subfigure}[b]{0.23\textwidth}
		\centering
		\includegraphics[width=\textwidth, height=2.5cm]{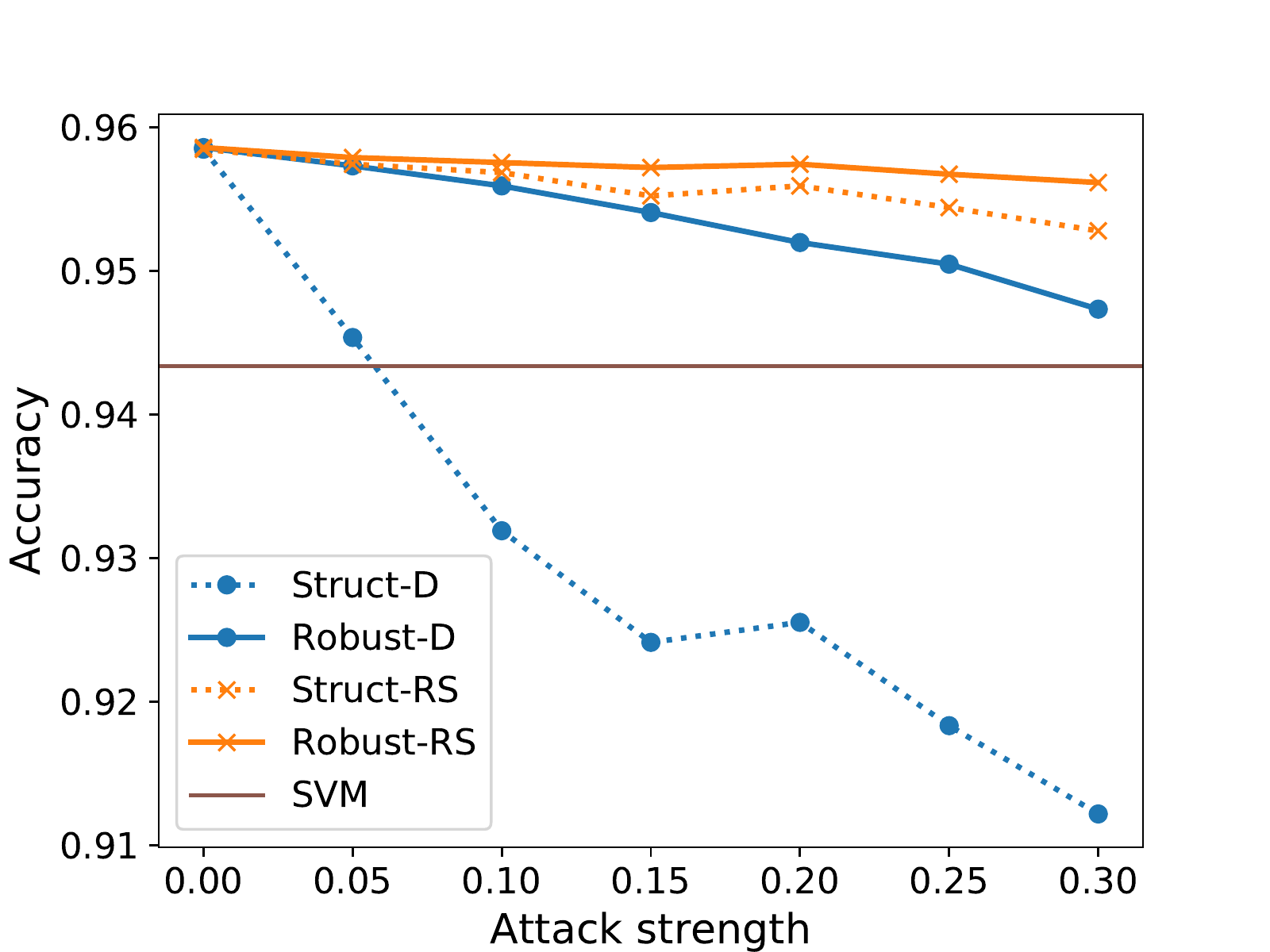}
		\caption{Reuters-H}
		\label{fig-Reuters-high-del}
	\end{subfigure}
	\hfill
	\begin{subfigure}[b]{0.23\textwidth}
		\centering
		\includegraphics[width=\textwidth, height=2.5cm]{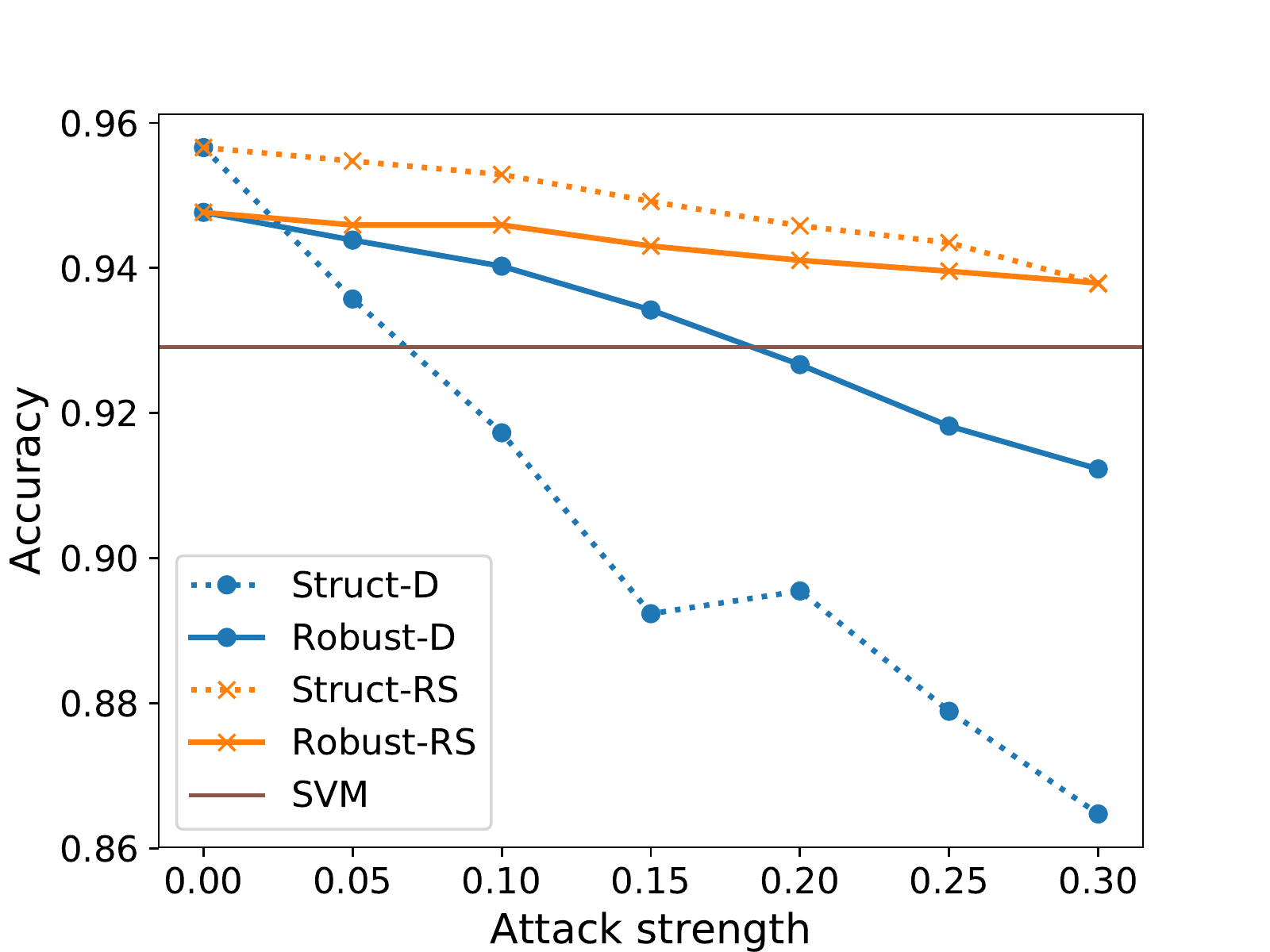}
		\caption{Reuters-L}
		\label{fig-Reuters-low-del}
	\end{subfigure}
	\hfill
	\begin{subfigure}[b]{0.23\textwidth}
		\centering
		\includegraphics[width=\textwidth, height=2.5cm]{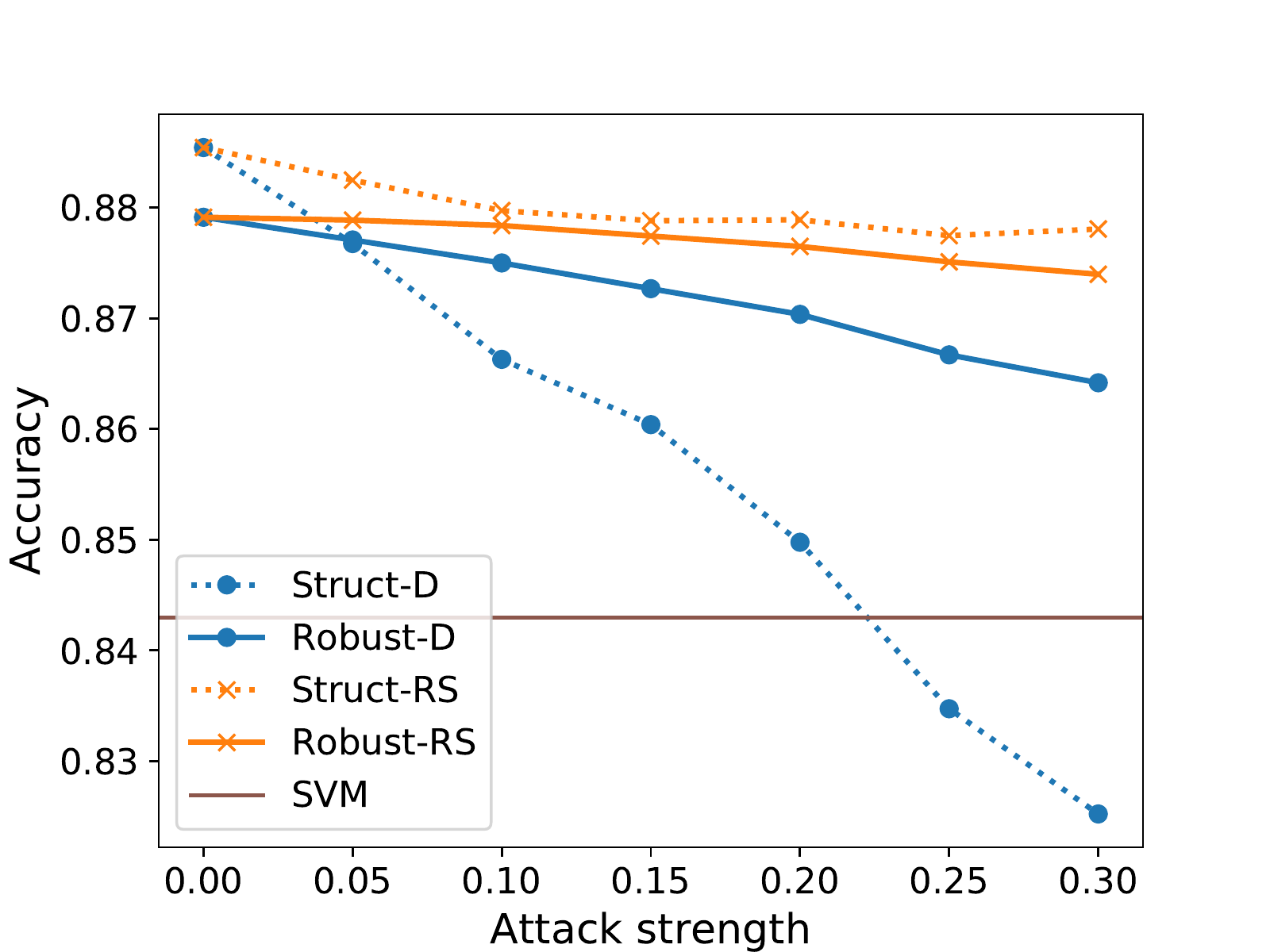}
		\caption{Cora-H}
		\label{fig-Cora-high-del}
	\end{subfigure}
	\hfill
	\begin{subfigure}[b]{0.23\textwidth}
		\centering
		\includegraphics[width=\textwidth, height=2.5cm]{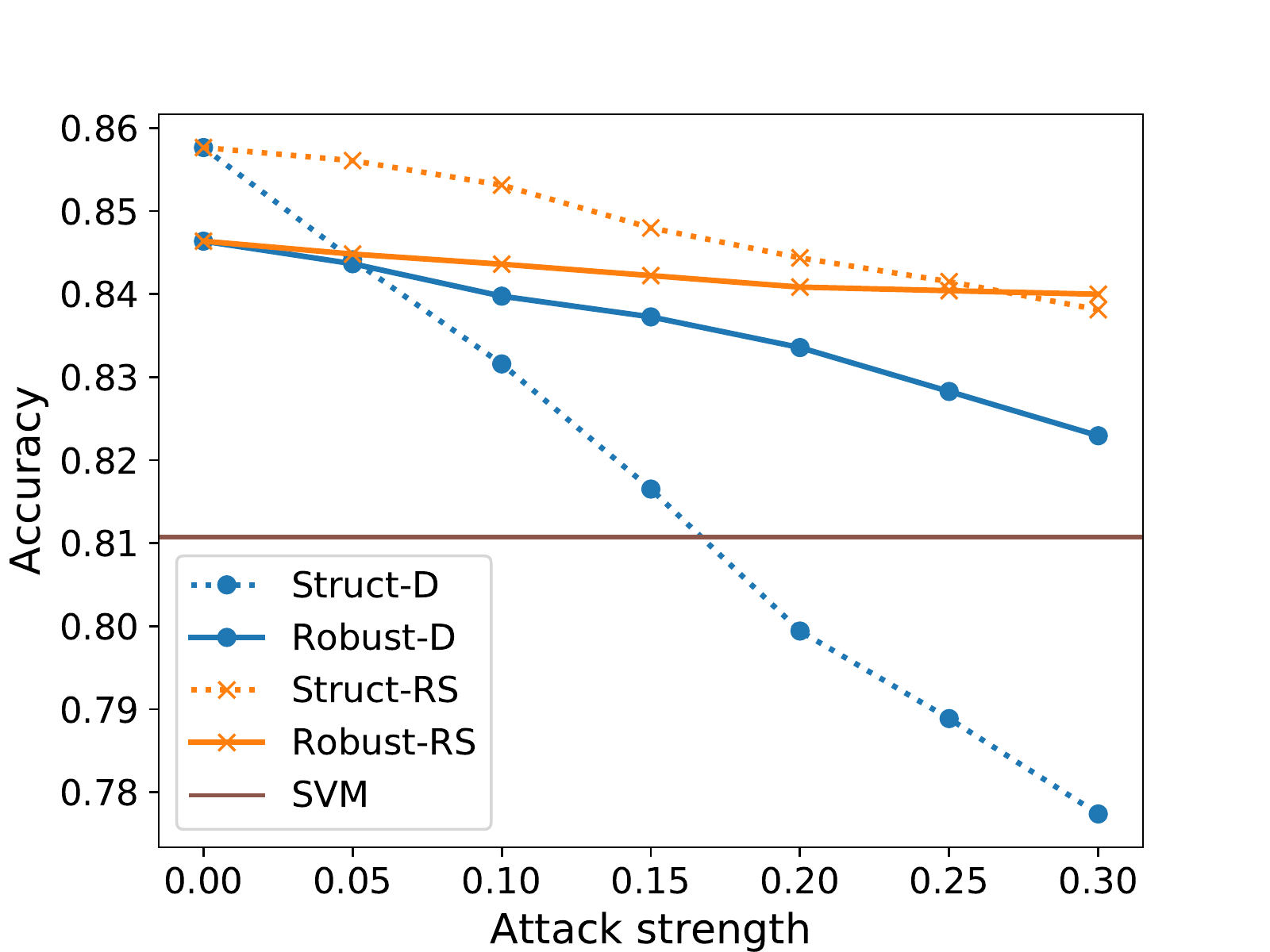}
		\caption{Cora-L}
		\label{fig-Cora-low-del}
	\end{subfigure}
	\hfill
	\begin{subfigure}[b]{0.23\textwidth}
		\centering
		\includegraphics[width=\textwidth, height=2.5cm]{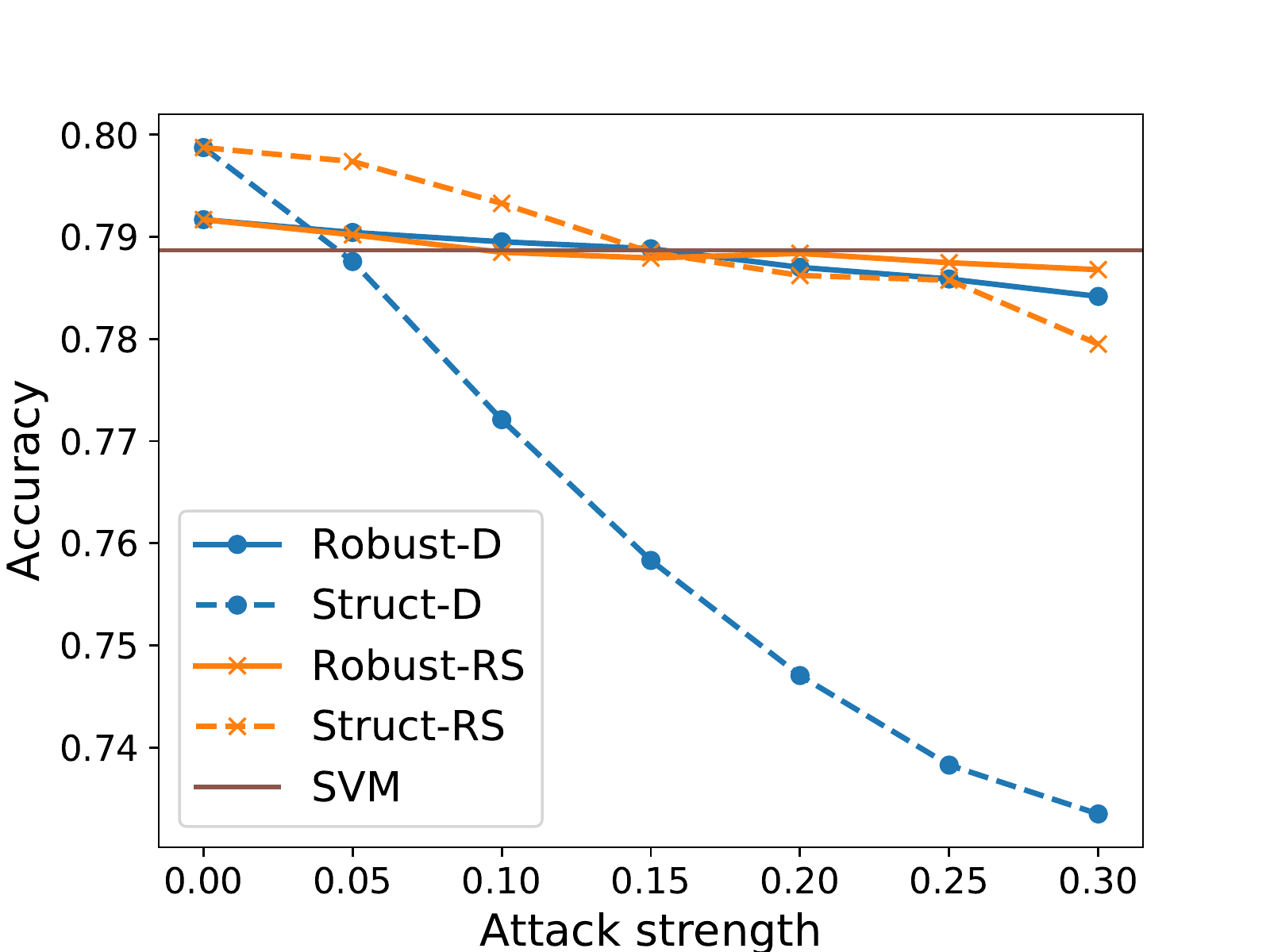}
		\caption{WebKB-H}
		\label{fig-Web-high-del}
	\end{subfigure}
	\hfill
	\begin{subfigure}[b]{0.23\textwidth}
		\centering
		\includegraphics[width=\textwidth, height=2.5cm]{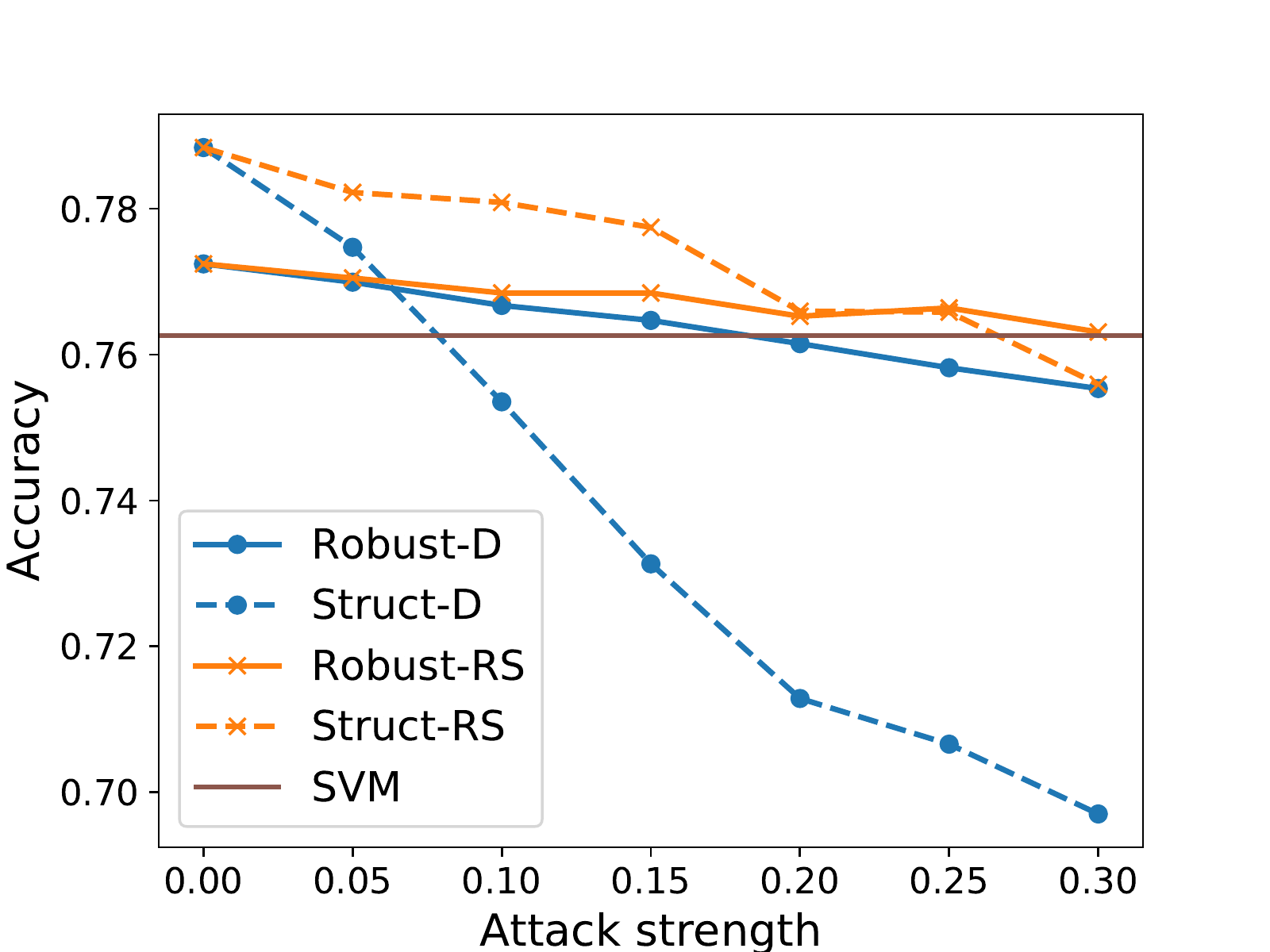}
		\caption{WebKB-L}
		\label{fig-Web-low-del}
	\end{subfigure}
	\hfill
	\begin{subfigure}[b]{0.23\textwidth}
		\centering
		\includegraphics[width=\textwidth, height=2.5cm]{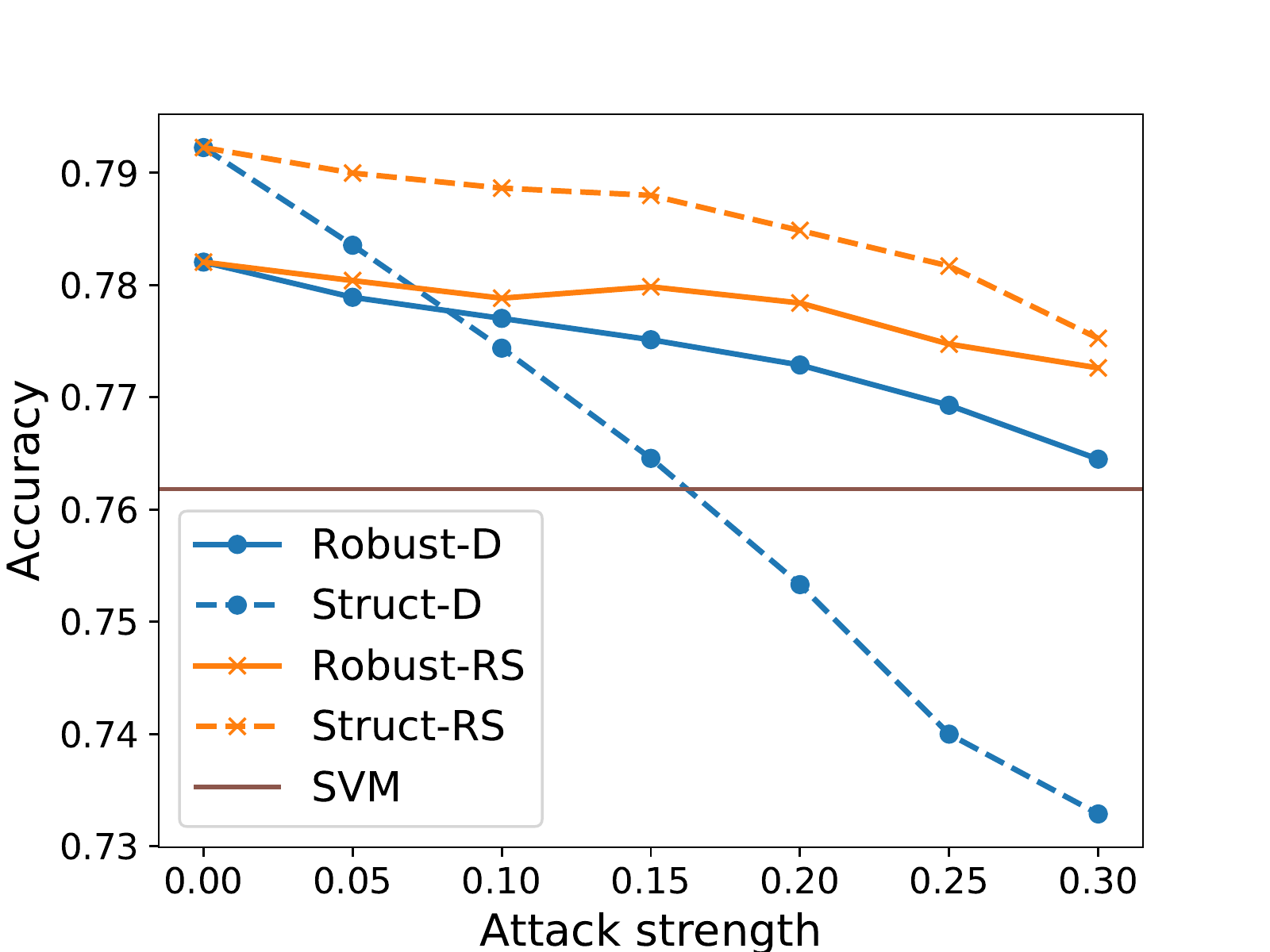}
		\caption{CiteSeer-H}
		\label{fig-Cite-high-del}
	\end{subfigure}
	\hfill
	\begin{subfigure}[b]{0.23\textwidth}
		\centering
		\includegraphics[width=\textwidth, height=2.5cm]{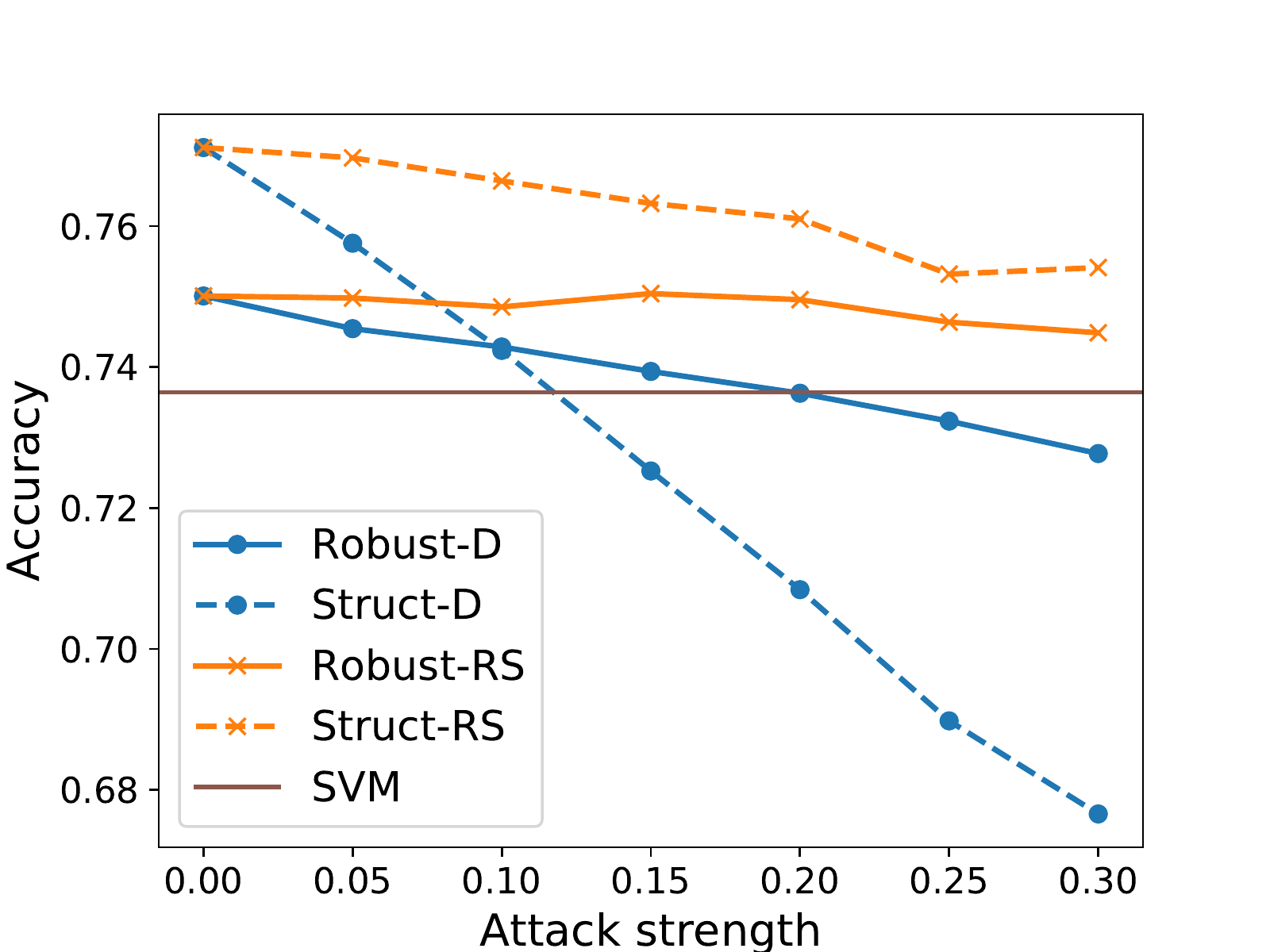}
		\caption{CiteSeer-L}
		\label{fig-Cite-low-del}
	\end{subfigure}
	\hfill
	\caption{Accuracies of AMN (dotted lines) and R-AMN (solid lines) under edge deletion attacks.  } 
	\label{fig-results-D}
\end{figure}

\subsection{ROBUSTNESS AGAINST STRUCTURAL ATTACKS} The AMN model jointly uses node features and structural information to do classification.
The impact of structural attacks depends on the importance of structure in classification, which in turn depends on the quality of information captured by node features.

\begin{figure}[t]
	\centering
	\begin{subfigure}[b]{0.23\textwidth}
		\centering
		\includegraphics[width=\textwidth, height=2.5cm]{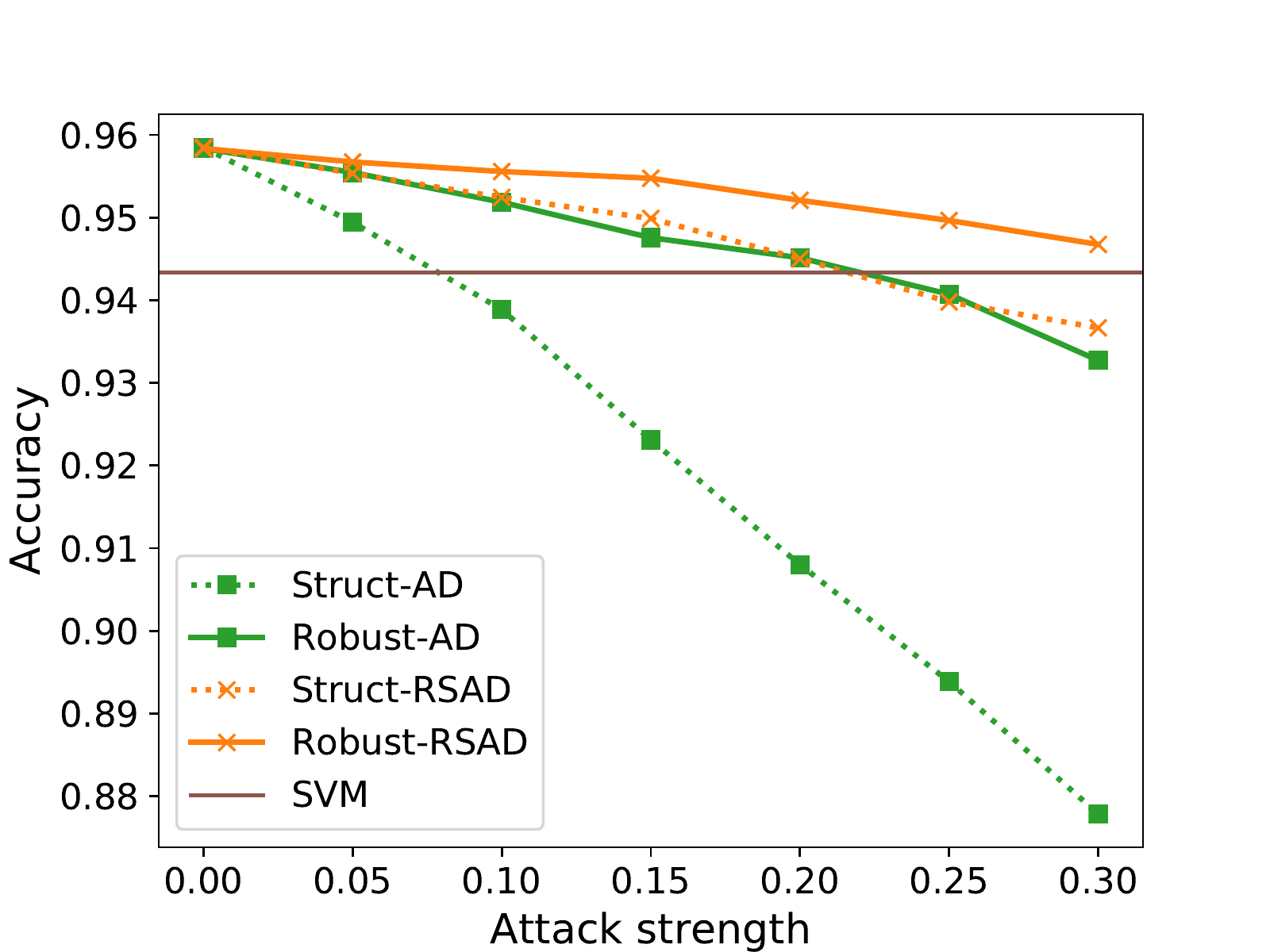}
		\caption{Reuters-H}
		\label{fig-Reuters-high-ad}
	\end{subfigure}
	\hfill	
	\begin{subfigure}[b]{0.23\textwidth}
		\centering
		\includegraphics[width=\textwidth, height=2.5cm]{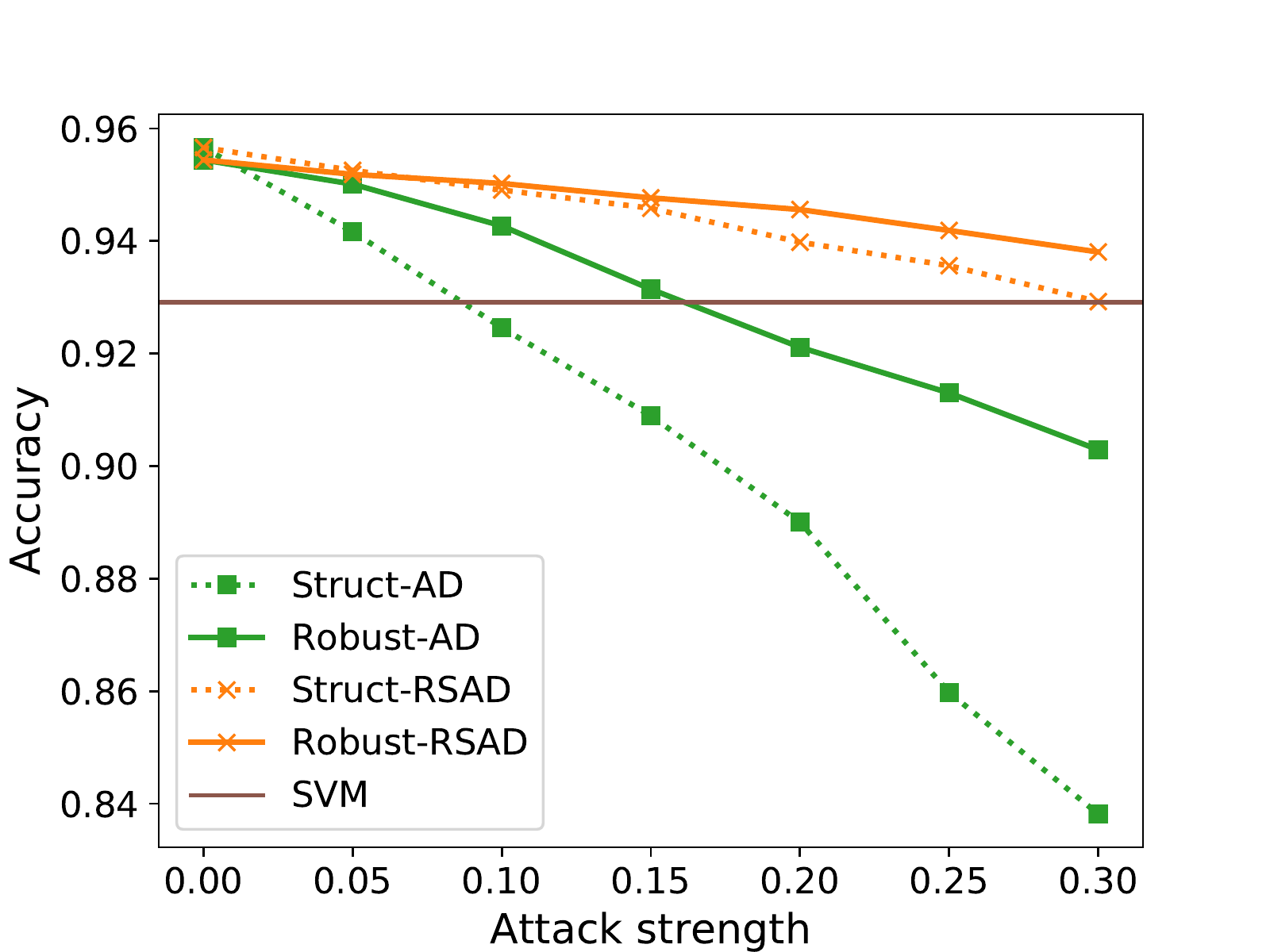}
		\caption{Reuters-L}
		\label{fig-Reuters-low-ad}
	\end{subfigure}
	\hfill	
	\begin{subfigure}[b]{0.23\textwidth}
		\centering
		\includegraphics[width=\textwidth, height=2.5cm]{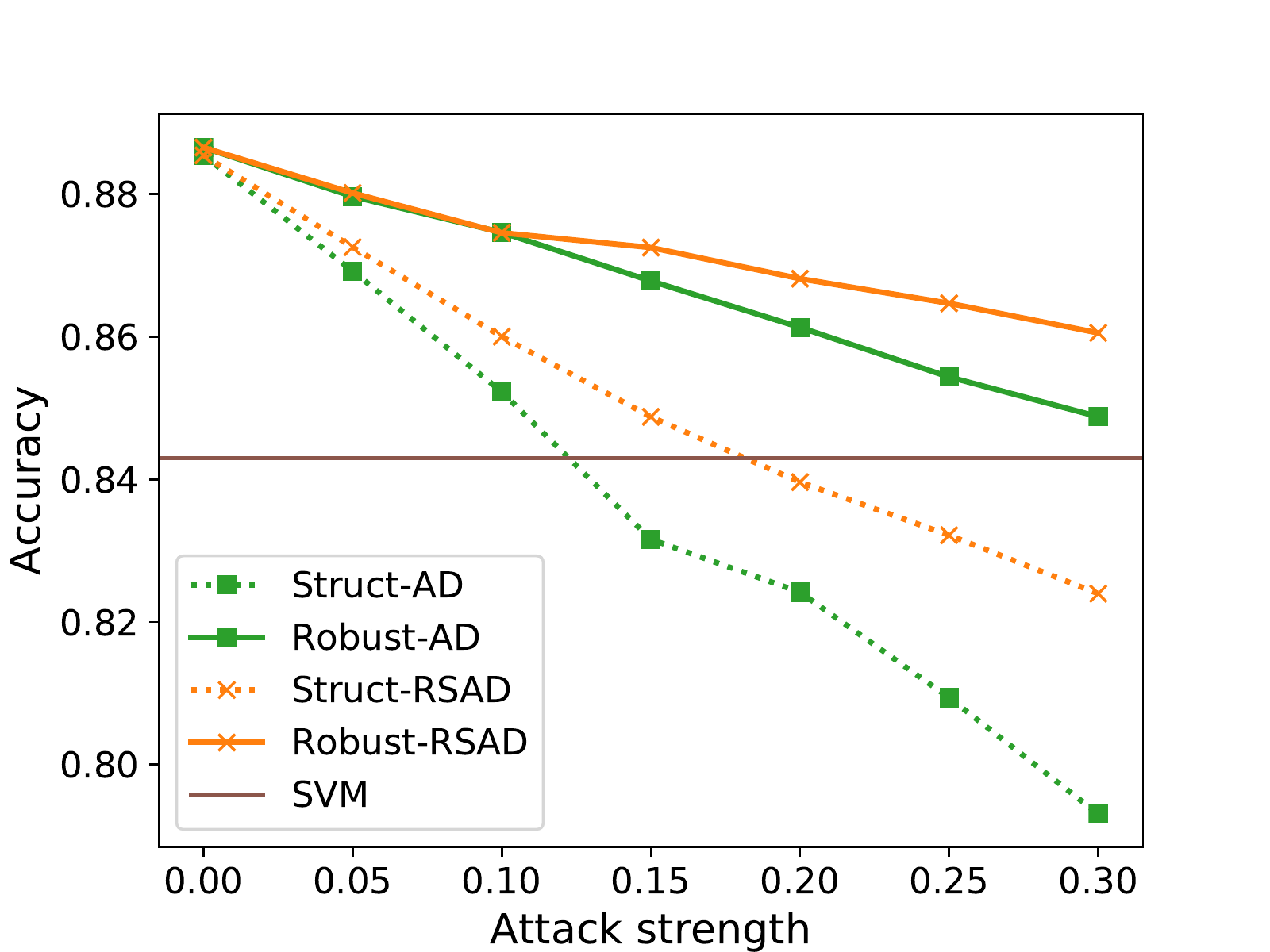}
		\caption{Cora-H}
		\label{fig-Cora-high-ad}
	\end{subfigure}
	\hfill	
	\begin{subfigure}[b]{0.23\textwidth}
		\centering
		\includegraphics[width=\textwidth, height=2.5cm]{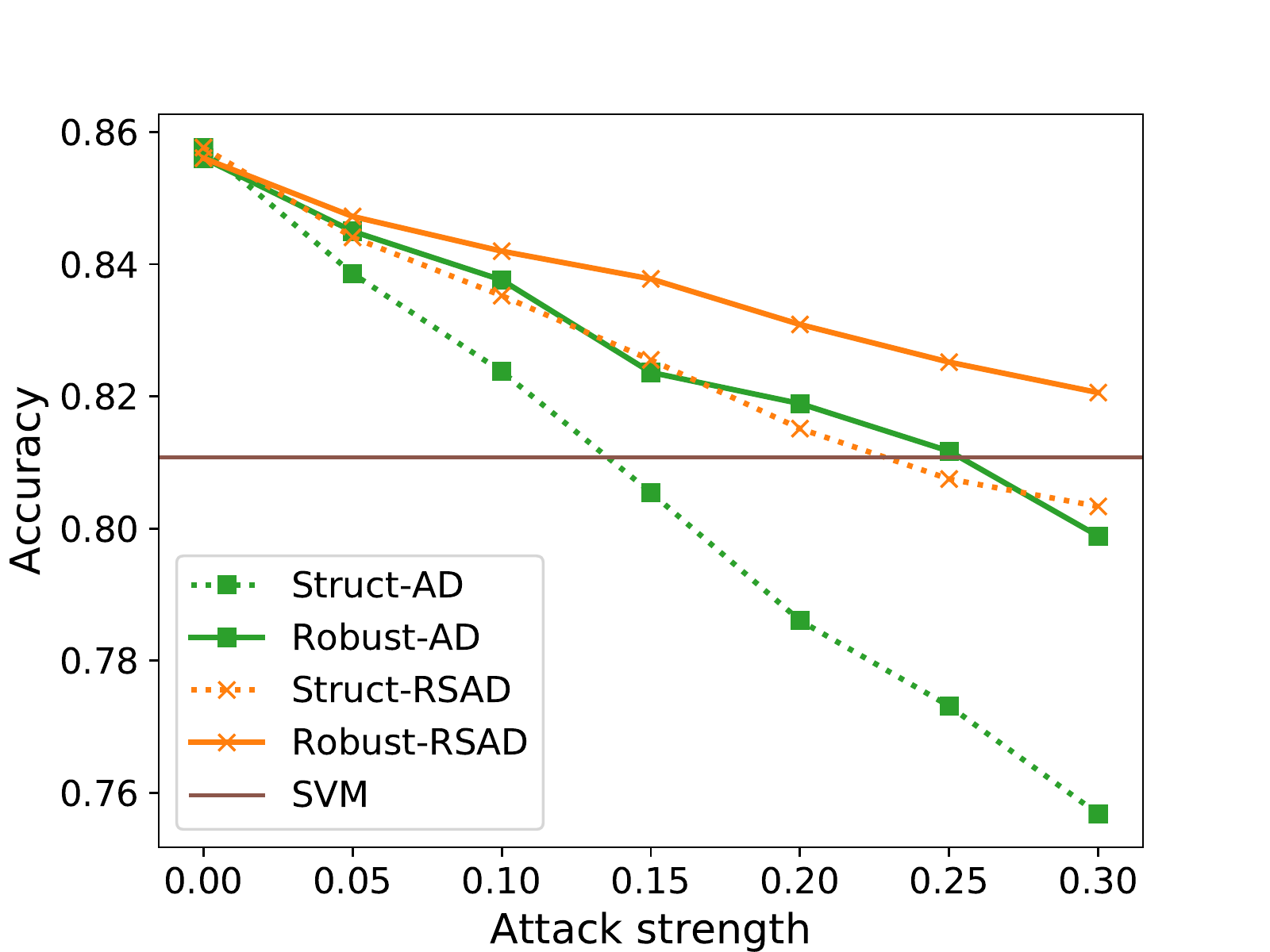}
		\caption{Cora-L}
		\label{fig-Cora-low-ad}
	\end{subfigure}
	\hfill
	\begin{subfigure}[b]{0.23\textwidth}
		\centering
		\includegraphics[width=\textwidth, height=2.5cm]{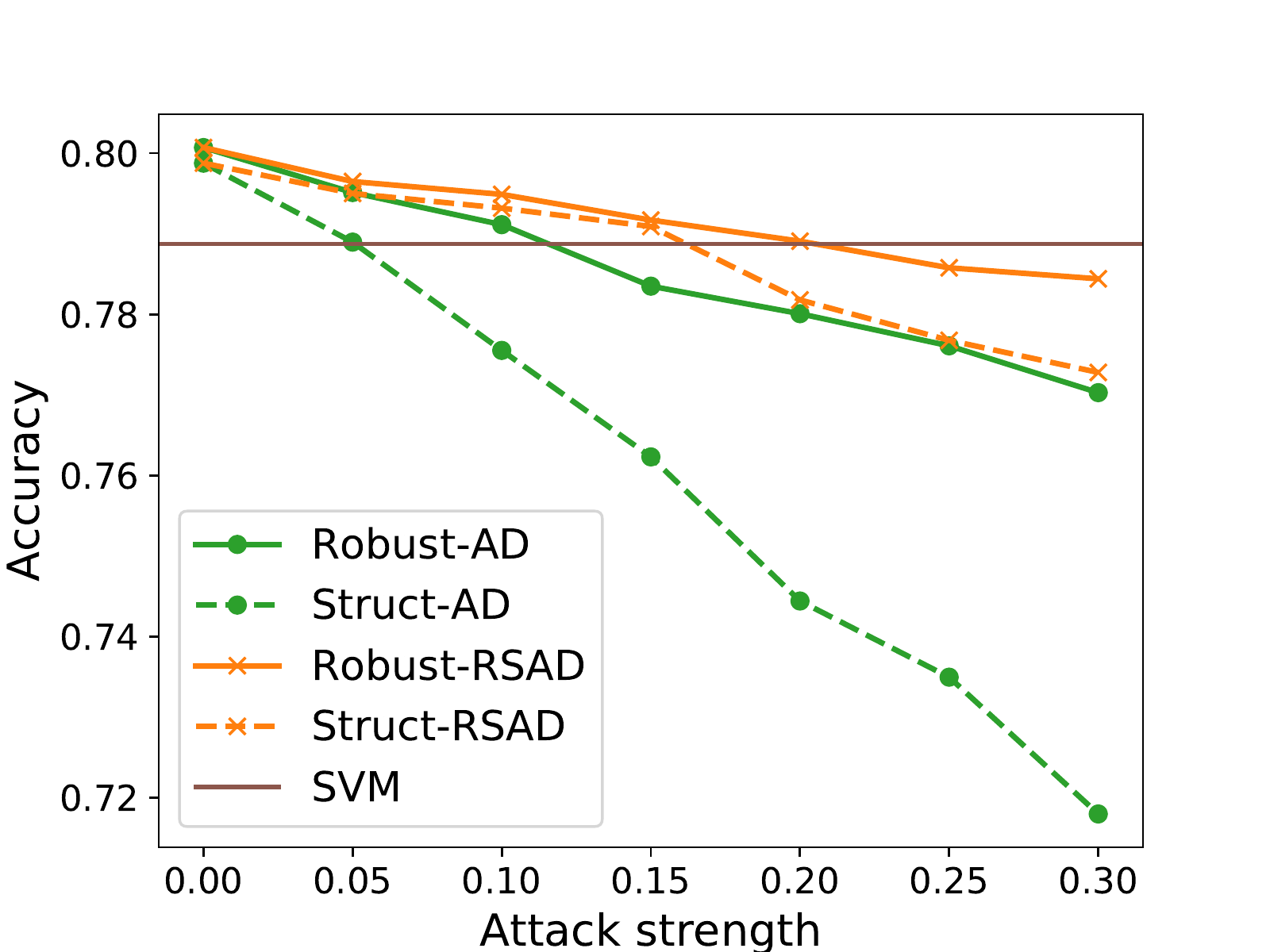}
		\caption{WebKB-H}
		\label{fig-Web-high-ad}
	\end{subfigure}
	\hfill	
	\begin{subfigure}[b]{0.23\textwidth}
		\centering
		\includegraphics[width=\textwidth, height=2.5cm]{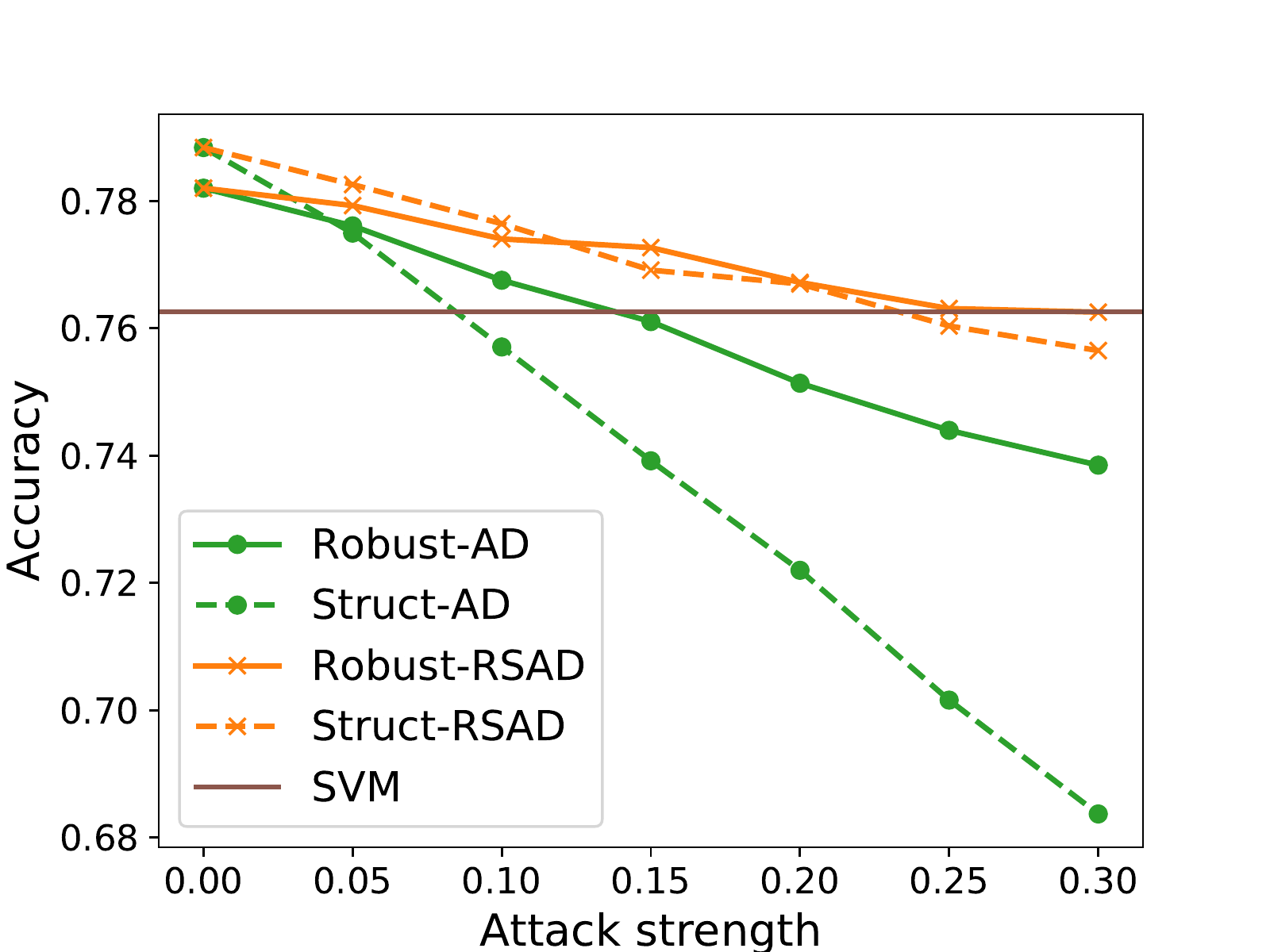}
		\caption{WebKB-L}
		\label{fig-Web-low-ad}
	\end{subfigure}
	\hfill	
	\begin{subfigure}[b]{0.23\textwidth}
		\centering
		\includegraphics[width=\textwidth, height=2.5cm]{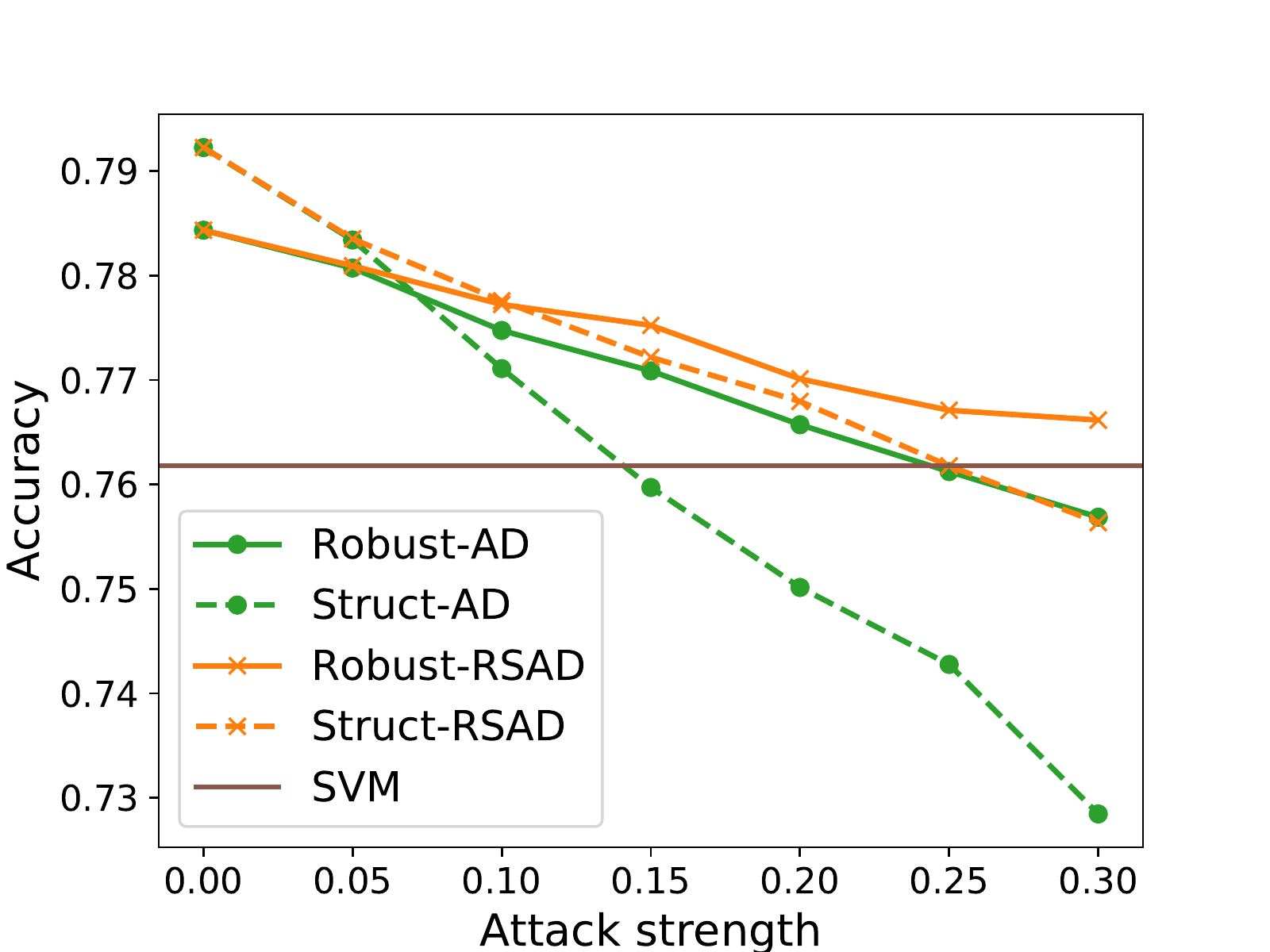}
		\caption{CiteSeer-H}
		\label{fig-Cite-high-ad}
	\end{subfigure}
	\hfill	
	\begin{subfigure}[b]{0.23\textwidth}
		\centering
		\includegraphics[width=\textwidth, height=2.5cm]{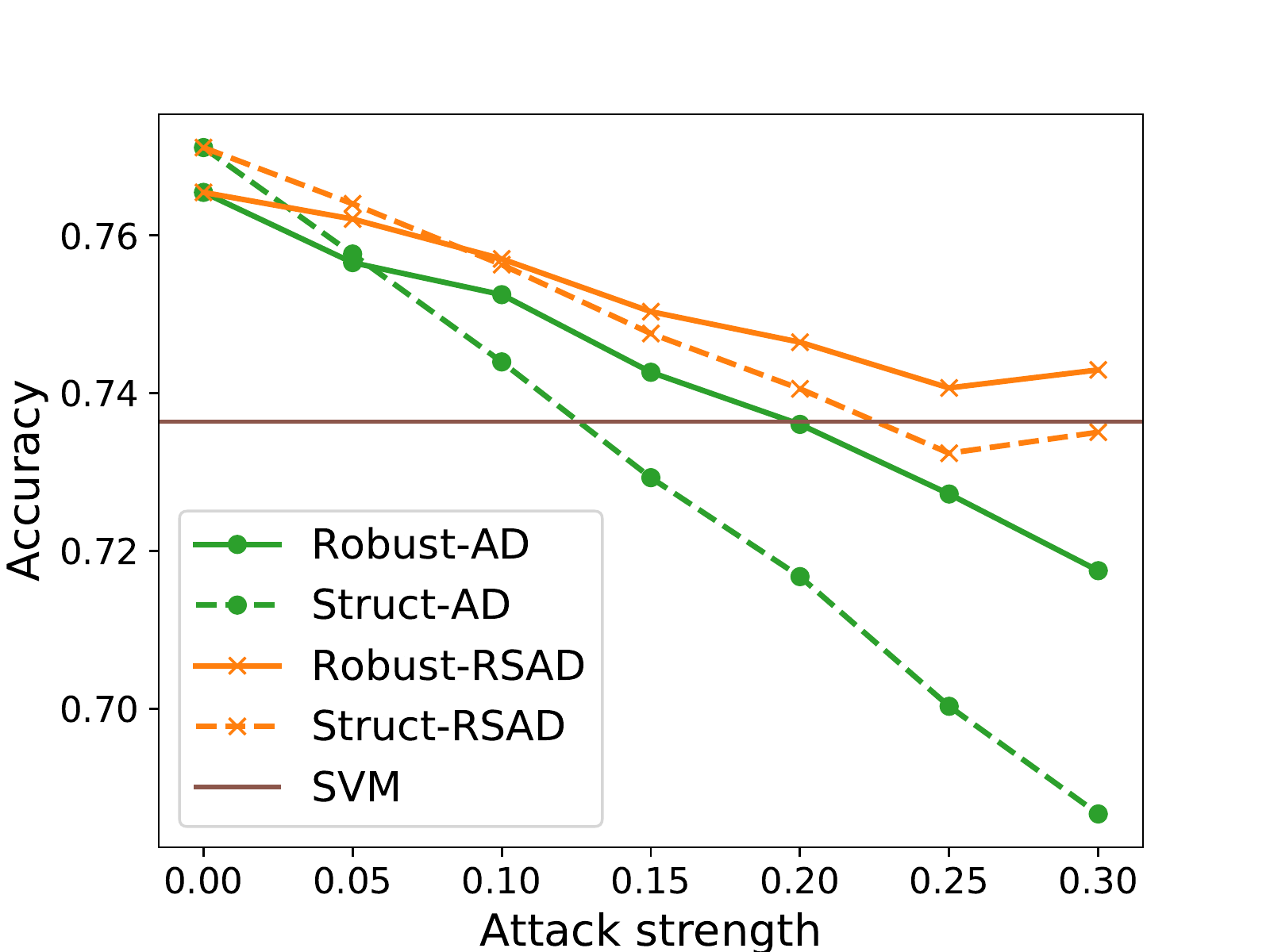}
		\caption{CiteSeer-L}
		\label{fig-Cite-low-ad}
	\end{subfigure}	
	\caption{Accuracies of AMN (dotted lines) and R-AMN (solid lines) under edge addition and deletion attacks } 
	\label{fig-results-AD}
\end{figure}

To study the impact of such node-specific information, we consider two settings:  high-discriminative (H-Dis), which uses more features, and low-discriminative  (L-Dis), which uses a smaller subset of features.
Specifically, for the Reuters dataset, we select $200$ top (in terms of frequencies) features in the H-Dis case (termed Reuters-H henceforth) and randomly selected $200$ out of the top $600$ features in the  L-Dis case (Reuters-L). For WebKB, we randomly select $200$ in L-Dis case (WebKB-L) and $300$ in H-Dis case (WebKB-H) out of the $1703$ features. For Cora, we use all the $1433$ features in H-Dis case (Cora-H) and randomly select $600$ features in the L-Dis case (Cora-L). For CiteSeer, we randomly select $500$ in L-Dis case (CiteSeer-L) and $1000$ in H-Dis case (CiteSeer-H) out of the $3703$ features. In our experiments, we use cross-validation on the training set to tune the two parameters of the R-AMN: the trade-off parameter $C$ and the adversarial budget $b$. We note that when tuning R-AMN, the defender has no knowledge of the strength of the attacker. We use a simulated attacker that can modify the validation set (in cross-validation) with $b = 0.1$, meaning that the attacker can change $10\%$ of the edges.

We consider a baseline attacker that can randomly add links between nodes belonging to different classes and  delete links between nodes with the same labels.
We term such an attack \textit{Struct-RSAD} (Remove Same and Add Different). In the case where the attacker is only allowed to delete links, we term the attack as \textit{Struct-RS}. We test R-AMN under four structural attacks: \textit{Struct-D} and \textit{Struct-AD} (our attacks, deleting links in the former, and adding or deleting in the latter), and \textit{Struct-RS} and \textit{Struct-RSAD} (heuristic baseline attacks above). We denote the performance of R-AMN exposed to these attacks as \textit{Robust-D}, \textit{Robust-AD}, \textit{Robust-RS}, \textit{Robust-RSAD}, respectively. We overload the notations by denoting the performance of AMN under the four attacks as \textit{Struct-D}, \textit{Struct-AD}, \textit{Struct-RS}, \textit{Struct-RSAD}, respectively.

Fig.~\ref{fig-results-D} and Fig.~\ref{fig-results-AD} show the
average accuracy (over $20$ independent data splits) of AMN (dotted
lines) and R-AMN (solid lines) under structural attacks as well as the
accuracy of a linear SVM classifier. First, by modifying a small
portion of the links in the graph, the accuracy of AMN drops below
that of SVM (which does not exploit relations), meaning that relations
among data points indeed introduce extra vulnerabilities.  Moreover,
structural attacks tend to be more severe when the node features are
less discriminative (the L-Dis case), where linking information plays
a relatively more important role in classification.  Notably, the
accuracy of R-AMN drops significantly slower than that of AMN under
structural attacks in all settings and stays above that of the SVM,
even when a relatively large fraction of the node connections are modified. These show that robust AMN preserves the benefits of using structural information even if network structure is maliciously modified.

\begin{figure}[!t]
	\begin{subfigure}[b]{0.23\textwidth}
		\centering
		\includegraphics[width=\textwidth, height=2.5cm]{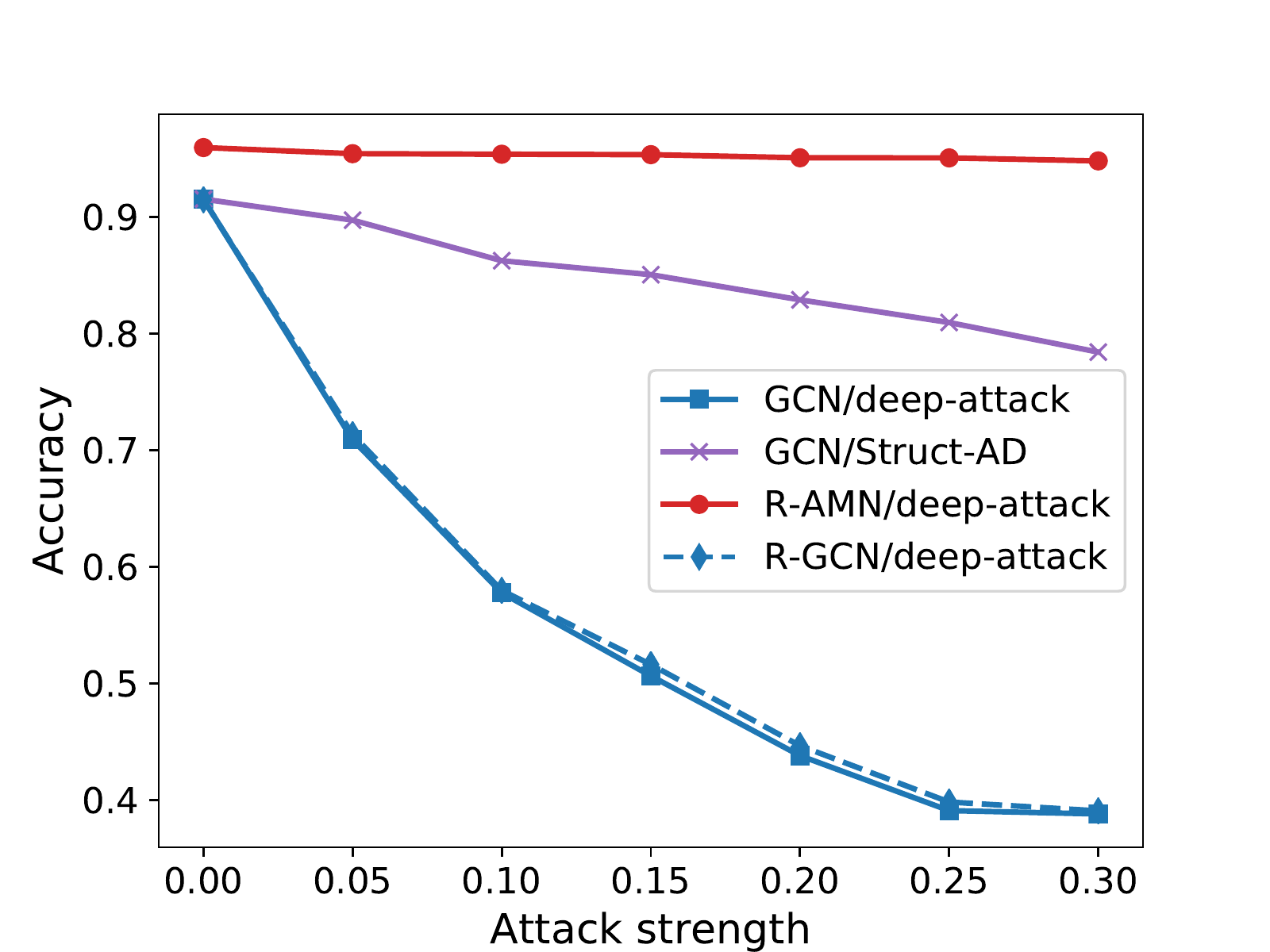}
		\caption{Reuters-H}
		\label{fig-deep-attack-Reuters-H}
	\end{subfigure}	
	\hfill
	\begin{subfigure}[b]{0.23\textwidth}
		\centering
		\includegraphics[width=\textwidth, height=2.5cm]{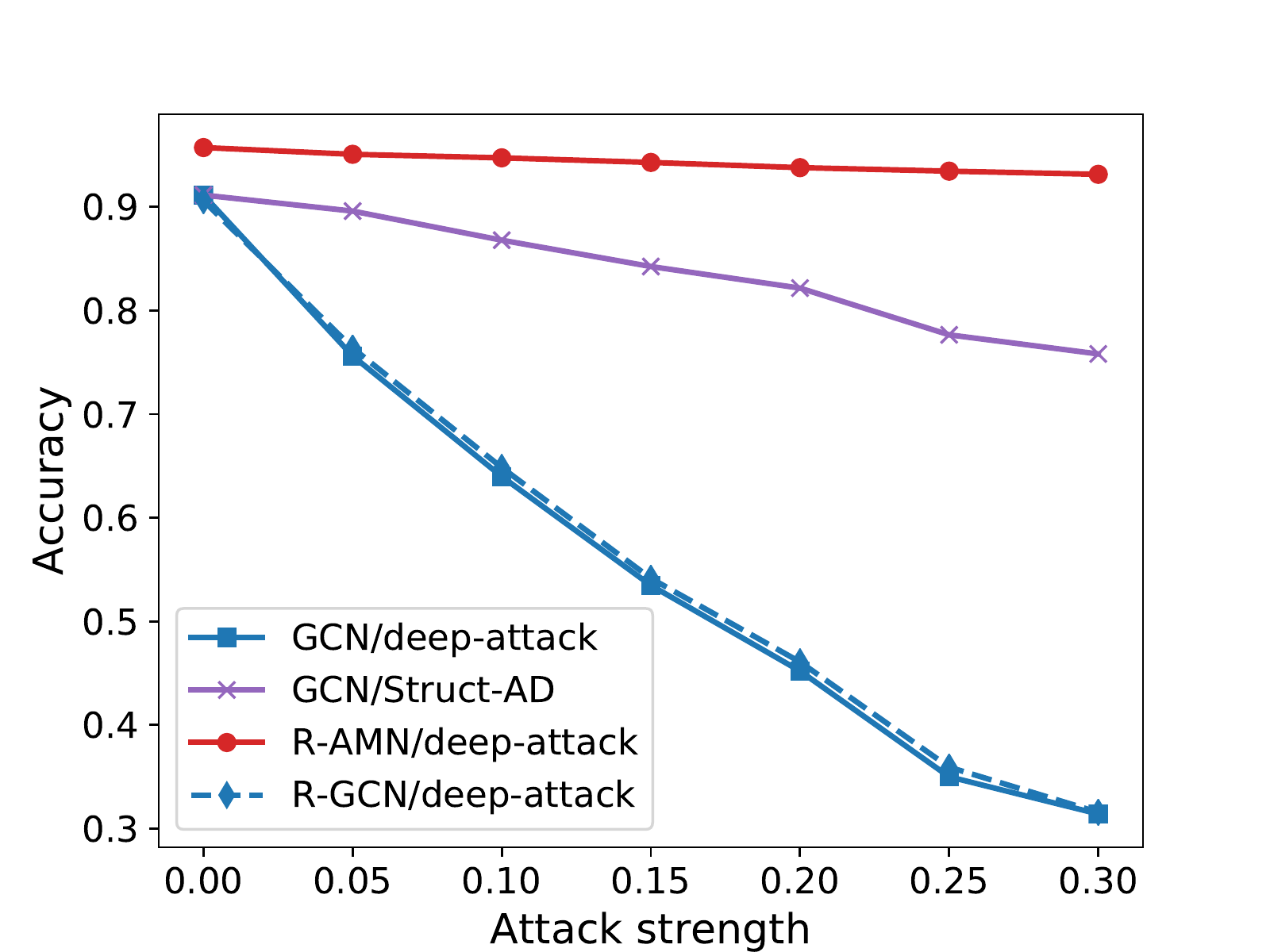}
		\caption{Reuters-L}
		\label{fig-deep-attack-Reuters-L}
	\end{subfigure}
	\hfill
	\begin{subfigure}[b]{0.23\textwidth}
		\centering
		\includegraphics[width=\textwidth, height=2.5cm]{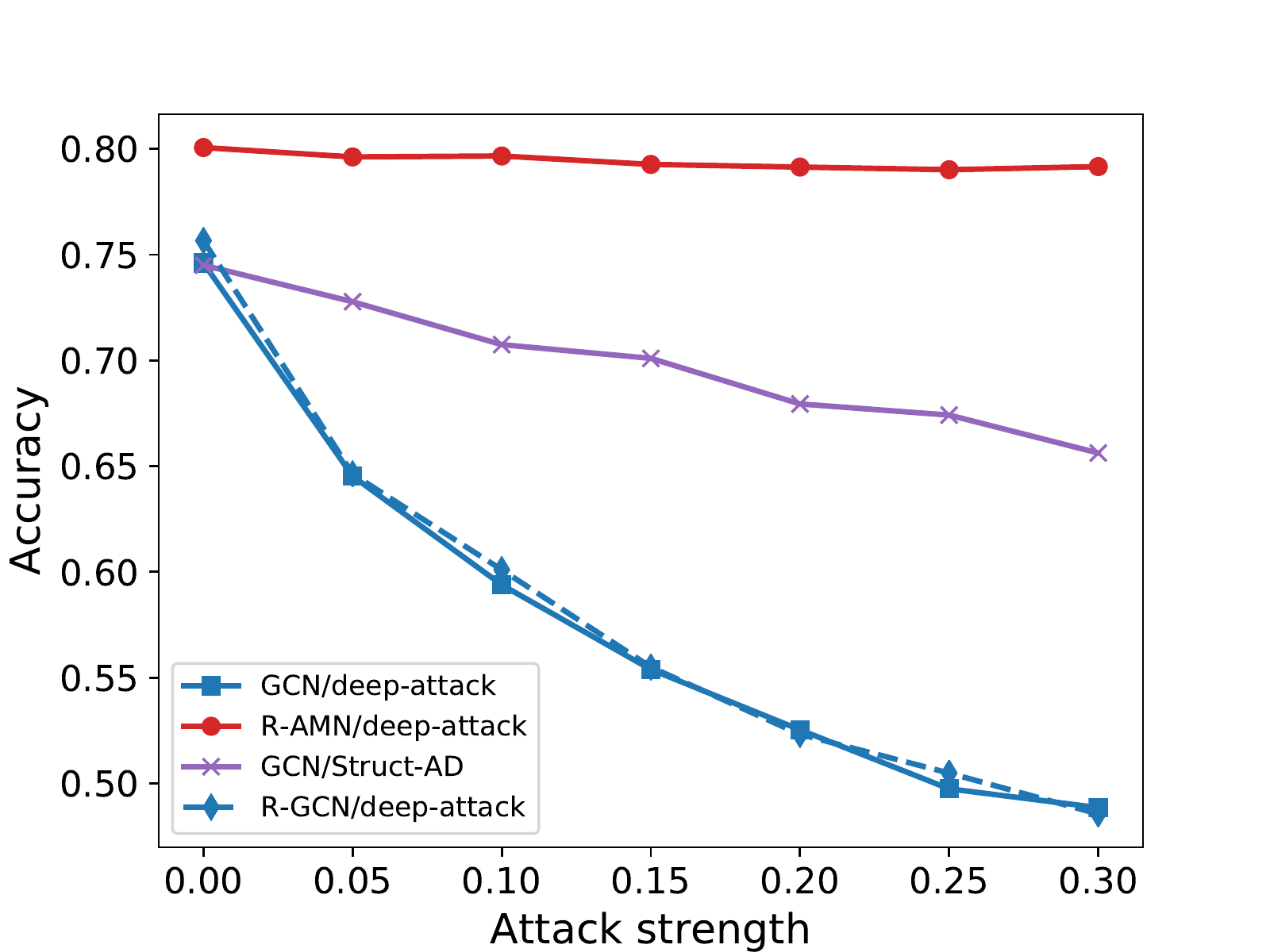}
		\caption{WebKB-H}
		\label{fig-deep-attack-Web-H}
	\end{subfigure}	
	\hfill
	\begin{subfigure}[b]{0.23\textwidth}
		\centering
		\includegraphics[width=\textwidth, height=2.5cm]{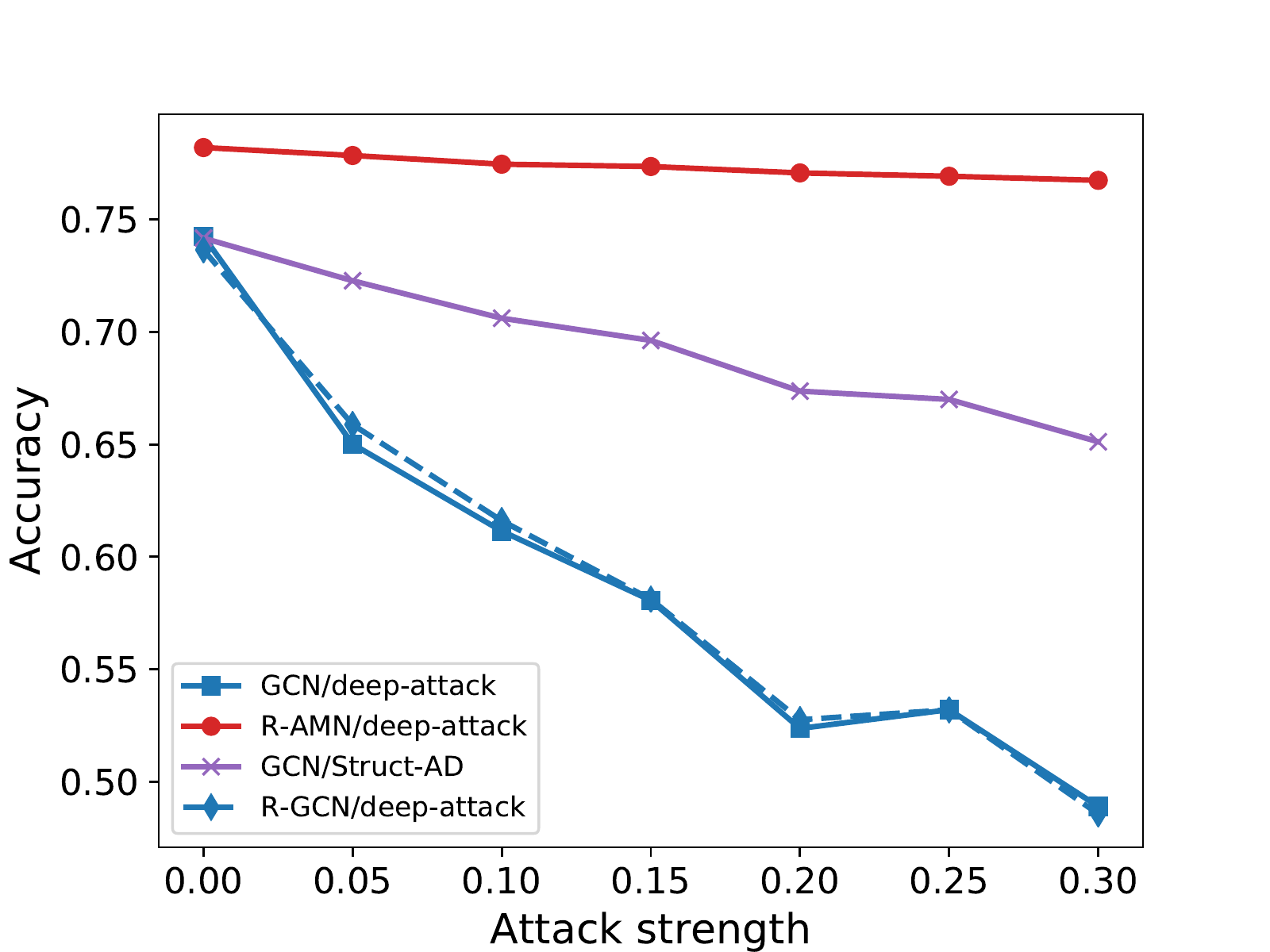}
		\caption{WebKB-L}
		\label{fig-deep-attack-Web-L}
	\end{subfigure}		
	\caption{R-AMN and R-GCN under deep-attack; GCN under deep-attack and Struct-AD attack.} 
	\label{fig-deep-results}
\end{figure}

\begin{figure}[t!]
	\begin{subfigure}[b]{0.23\textwidth}
		\centering
		\includegraphics[width=\textwidth, height=2.5cm]{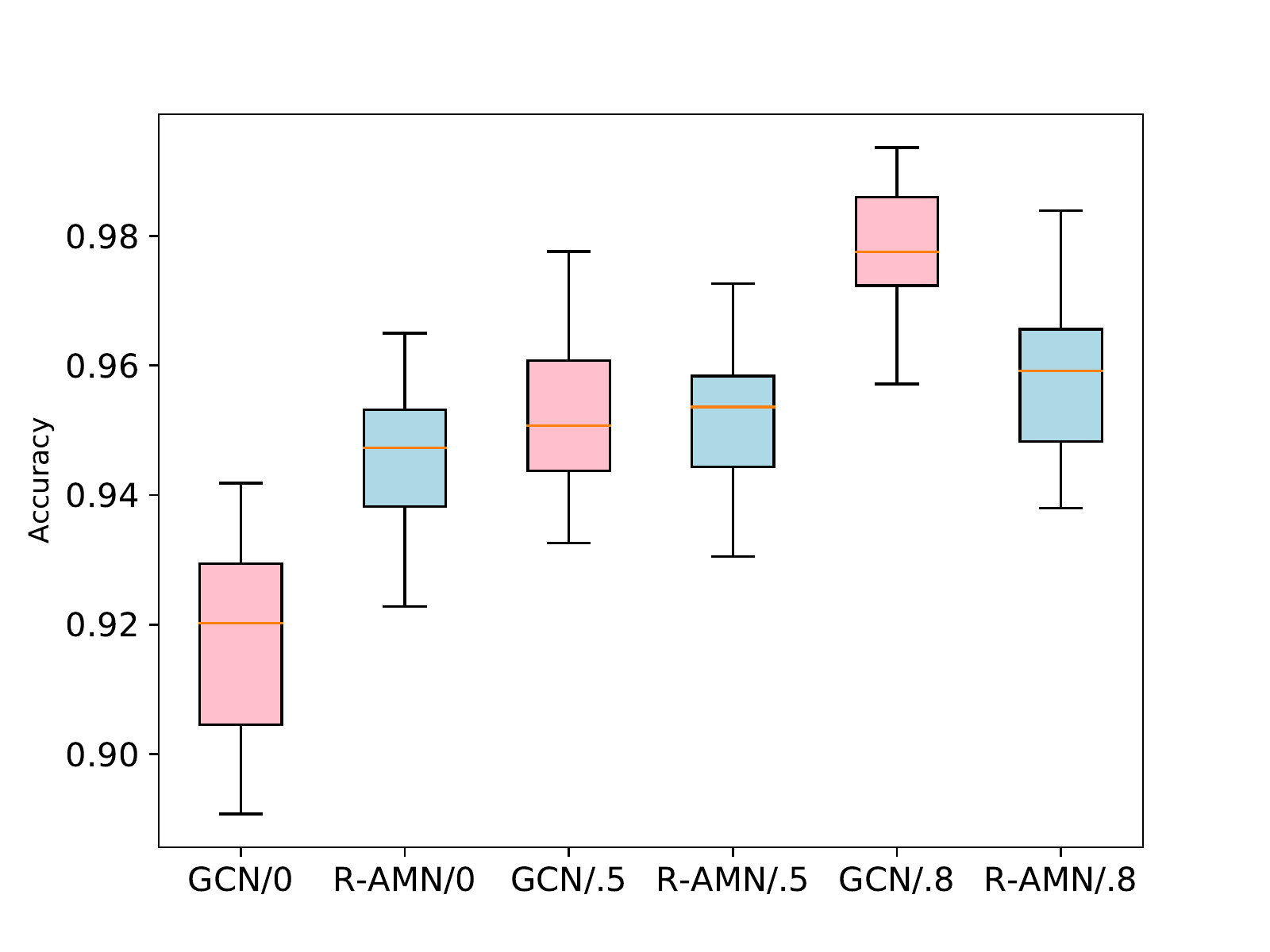}
		\caption{Reuters-H}
		\label{fig-reuters-H-compare}
	\end{subfigure}
	\hfill
	\begin{subfigure}[b]{0.23\textwidth}
		\centering
		\includegraphics[width=\textwidth, height=2.5cm]{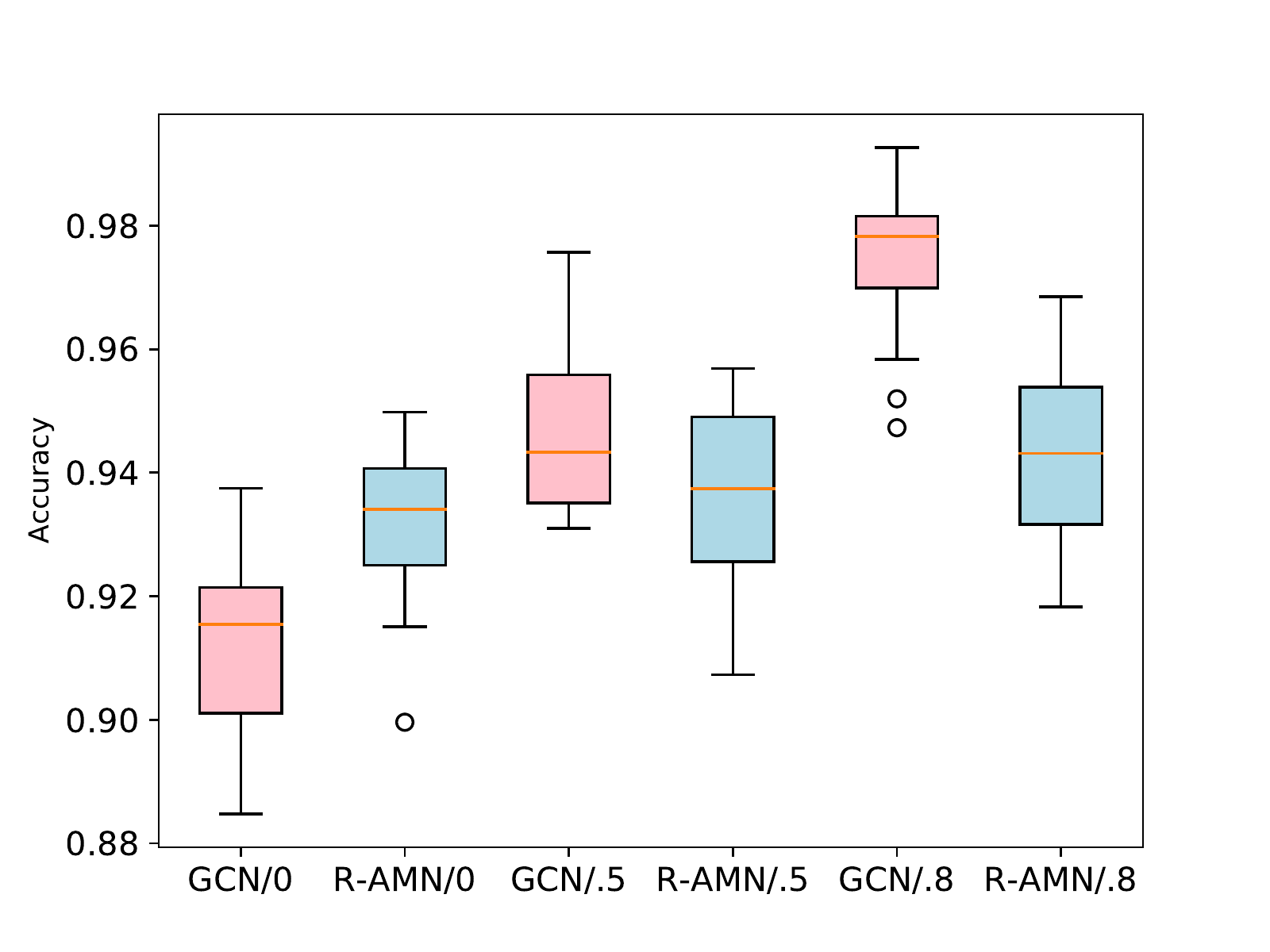}
		\caption{Reuters-L}
		\label{fig-reuters-L-compare}
	\end{subfigure}	
	\hfill
	\begin{subfigure}[b]{0.23\textwidth}
		\centering
		\includegraphics[width=\textwidth, height=2.5cm]{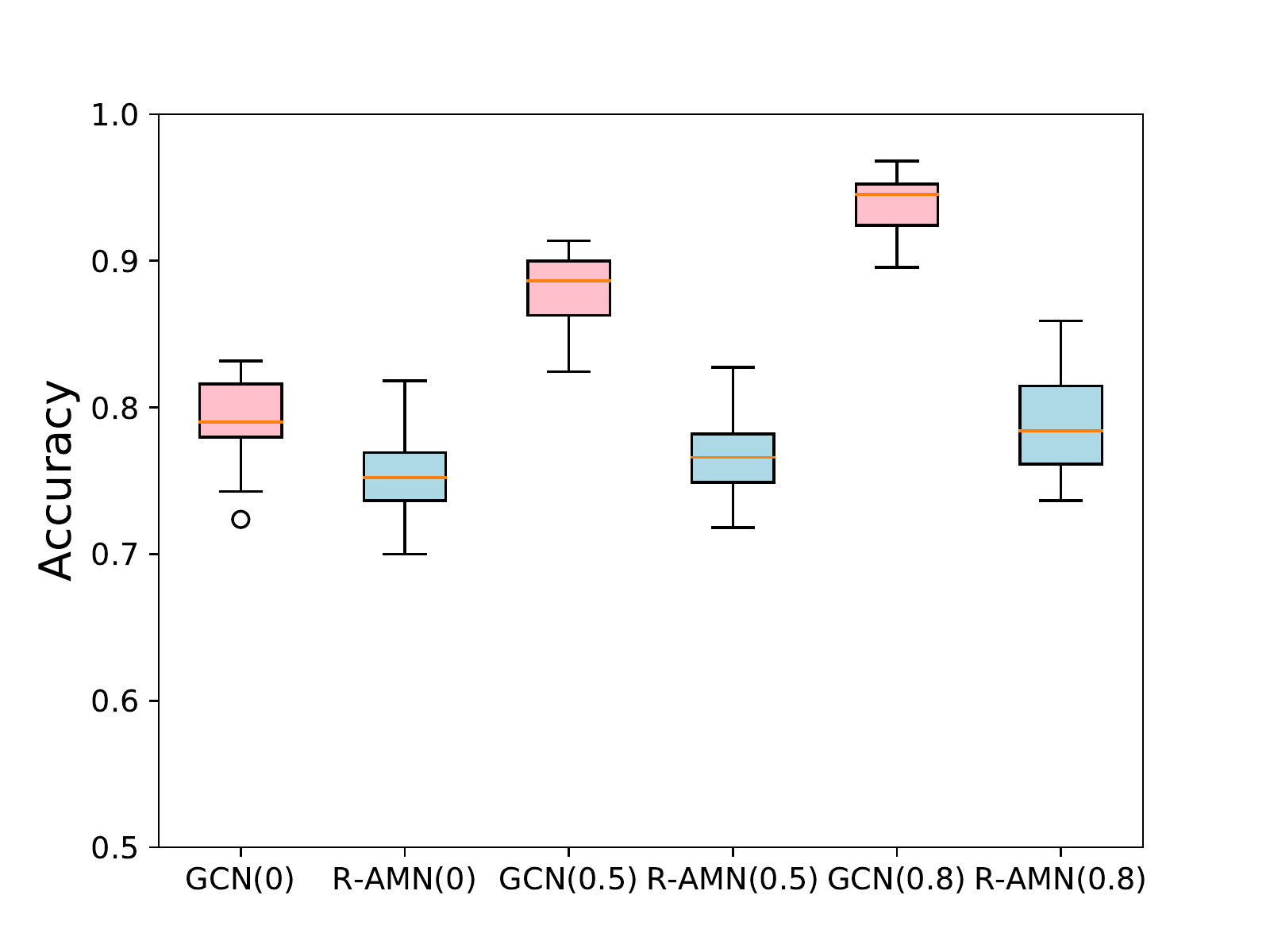}
		\caption{WebKB-H}
		\label{fig-Web-H-compare}
	\end{subfigure}
	\hfill
	\begin{subfigure}[b]{0.23\textwidth}
		\centering
		\includegraphics[width=\textwidth, height=2.5cm]{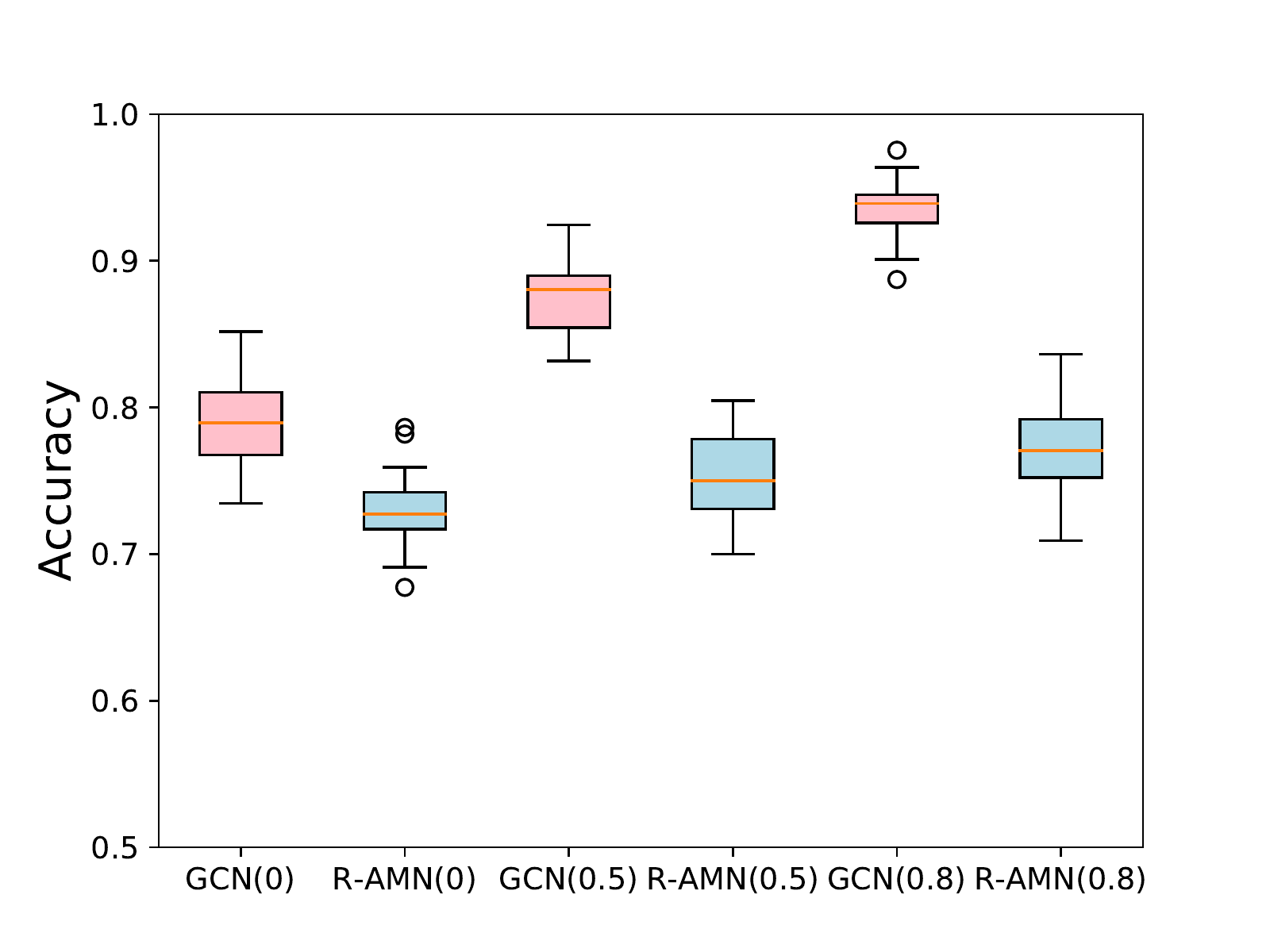}
		\caption{WebKB-L}
		\label{fig-Web-L-compare}
	\end{subfigure}		
	\caption{R-AMN and GCN on non-adversary data as graphs are purified, e.g. R-AMN($0.5$) stands for R-AMN when noisy edges are deleted with probability $0.5$.} 
	\label{fig-compare}
\end{figure}

\subsection{ROBUSTNESS AGAINST DEEP LEARNING BASED ATTACKS}
Having observed that our approach for robust AMN is indeed robust against our newly designed attacks on AMN, it is natural to wonder whether robust AMN remains robust against recent attacks on graph convolutional networks (GCNs).
The answer is non-obvious: on the one hand, attacks on GCN may not transfer to AMN---although there is ample prior evidence that attacks do often transfer from one learning approach to another~\citep{vorobeychik2018adversarial}; on the other hand, attacks on GCN target transductive learning, and as such also \emph{poison} the data on which the AMN would be trained.
To this end, we use a recent 
structural attack on GCN (termed deep-attack) proposed in \citep{DBLP:journals/corr/abs-1902-08412}, which has demonstrated impressive attack performance and transferability to neural network embedding-based approaches.  We thus test our R-AMN against this deep-attack, compared to GCN.
In addition, we compare the performance of R-AMN with a GCN based classifier \citep{kipf2016semi} on non-adversarial data.

GCN classifies nodes in a transductive setting, where labeled and unlabeled nodes reside in the \textit{same} graph; while R-AMN is trained over a \textit{training graph}, and makes predictions over an unseen \textit{test graph}.
To adapt R-AMN to the transductive setting,  we use deep-attack (with the same configurations as in \citep{DBLP:journals/corr/abs-1902-08412}, e.g., strongest ``Meta-Self'' mode and $0.1/0.1/0.8$ train/validation/test split) to modify the training graph and test graph separately (these can be subsets of the same larger graph, as is the case in transductive learning). Then we train R-AMN on the \textit{attacked} training graph and test it on the \textit{attacked} test graph. We also test the performance of GCN under deep-attack and our proposed structural attack \textit{Struct-AD} on the test graph. In addition, we test a robust GCN model (termed R-GCN) under deep-attack, which is based on adversarial training approach proposed in \citep{xu2019topology} on the test graph. The accuracies of R-AMN, GCN, and R-GCN under these attacks are presented in Fig.~\ref{fig-deep-results} for the Reuters and WebKB datasets (the appendix presents similar results for the Cora and CiteSeer datasets).
The main observation is that where R-GCN cannot defend against such deep-attack in a transductive setting, our proposed R-AMN is essentially invariant under deep-attack, in contrast to GCN, which is highly vulnerable to this attack, and also quite vulnerable to our proposed structural attacks aimed at AMN. 

Finally, we compare R-AMN and GCN on non-adversarial data, which are
evenly split into training and test graphs. 
We note that R-AMN ignores the links between the training and test
(sub-)graphs. 
The results are shown in Fig.~\ref{fig-compare}.
To interpret these, consider first GCN/0 and R-AMN/0 bars, which
correspond to the direct performance comparison on the given data.
We can observe that the difference in accuracy between R-AMN and GCN
in these cases is either small (on WebKB data) or, in fact, R-AMN
actually outperforms GCN (on Reuters data).
It's this latter observation that is surprising.
The reason is that Reuters data contains $\sim 10\%$ of noisy links,
that is, links connecting pairs of nodes with different labels.
This can be viewed as another symptom of ragility of GCN, but in any
case, we next consider what happens when we remove each noisy link
with some probability ($p=0.5$ or $p=0.8$).
The results are presented as bars with R-AMN/$p$ and GCN/$p$, where
$p$ corresponds to this probability of removed noisy links, and, as
expected, GCN performance improves as we improve data quality.
This improvement is significant when the graph is not particularly
noisy (WebKB), but the gap between R-AMN and GCN remains relatively
small when enough noisy links remain (Reuters, as well as Cora and CiteSeer; see the appendix, Fig.~6).
This further attests to greater robustness of R-AMN, but does exhibit
some cost in terms of accuracy on non-adversarial data, if this data
is sufficiently high quality.	

\section{CONCLUSION}
	
	We study robustness of the associative Markov network classifier under
	test-time attacks on network structure, where an attacker can delete and/or add links in the underlying graph.
	We formulate the task of robust
	learning as a bi-level program and propose an approximation algorithm
	to efficiently solve it.  Our experiments on real-world datasets
	demonstrate that the performance of robust AMN  degrades gracefully
	even under large adversarial modifications of the graph structure,
	preserving the advantages of using structural information in
	classifying relational data. We additionally compare robust AMN with
	the state-of-the-art deep learning based approaches in the
	transductive setting and demonstrate that robust AMN is significantly
	more robust to structural perturbations compared to deep graph embedding
	methods while sacrificing little performance on non-adversarial data, except when network data is of extremely high quality (a rarity in practice).
	
	\subsubsection*{Acknowledgements}
	
	This work was partially supported by the National Science Foundation (grants IIS-1905558 (CAREER) and IIS-1903207) and Army Research Office (grants W911NF1810208 (MURI) and W911NF1910241).
	
\bibliographystyle{abbrvnat}	
\bibliography{citation}

\newpage
\appendix

\onecolumn
\section*{Appendix}
\section{Formulation of Convex Quadratic Program}
We explicitly write out the convex quadratic program for learning robsut AMN, which is omitted in the main paper. By LP duality, we can replace the attacker's maximization problem using its dual minimization problem, which is further integrated into Eqn.~(5). Consequently, we can approximate Eqn.~(5) by the following convex quadratic program:

\begin{align}
\label{eqn-robust-qp}
&\min\  \frac{1}{2}||\mathbf{w}||^2 + C(N - \sum_{i=1}^N \sum_{k=1}^K \mathbf{w}_n^k \mathbf{x}_i \hat{y}_i^k  + \sum_{i=1}^N t_i + \sum_{(i,j)\in E} p_{ij} - D^- \cdot t_D) \nonumber \\
&\text{s.t.}\quad \forall i,k,\quad  t_i -\sum_{(i,j),(j,i) \in E} t_{ij}^k - \mathbf{w}_n^k\mathbf{x}_i + \hat{y}_i^k \geq 0, \nonumber \\
&\forall (i,j)\in E,k, \quad s_{ij}^k + t_{ij}^k + t_{ji}^k - w_e^k \geq 0,\  s_{ij}^k, t_{ij}^k,t_{ji}^k \geq 0, \nonumber\\
&\forall (i,j) \in E,\quad  p_{ij} - \sum_{k=1}^K s_{ij}^k - t_D + \sum_{(i,j)\in E} \sum_{k=1}^K w_e^k \hat{y}_i^k \hat{y}_j^k \geq 0,
\ p_{ij}, t_D \geq 0.
\end{align}
The minimization is over the weights $\mathbf{w}$ and the dual variables $t_i, p_{ij}, s_{ij}^k, t_{ij}^k, t_{ji}^k, t_D$.

\section{Additional Experiment Results}
We compare R-AMN and GCN under the deep-attack as well as on non-adversarial data on the Cora and CiteSeer datasets in the same experiment settings as in the main paper. Specifically, in Fig.~\ref{fig-deep-results}, "R-AMN/deep-attack" shows the accuracies of R-AMN under deep-attack with various degrees of graph perturbations, where the train graph and test graph are attacked by deep-attack separately. It demonstrates that R-AMN is robust to deep-attack even with relatively large structural perturbations. "GCN/deep-attack" and "GCN/Struct-AD" show the accuracies of GCN under deep-attack and our proposed \textit{Struct-AD} attack, respectively. Generally, deep-attack is a much more effective method to attack GCN models.
Fig.~\ref{fig-deep-compare} demonstrated that on non-adversarial data, the performances of R-AMN and GCN are comparable.

\begin{figure}[!hb]
	\begin{subfigure}[b]{0.24\textwidth}
		\centering
		\includegraphics[width=\textwidth, height=2cm]{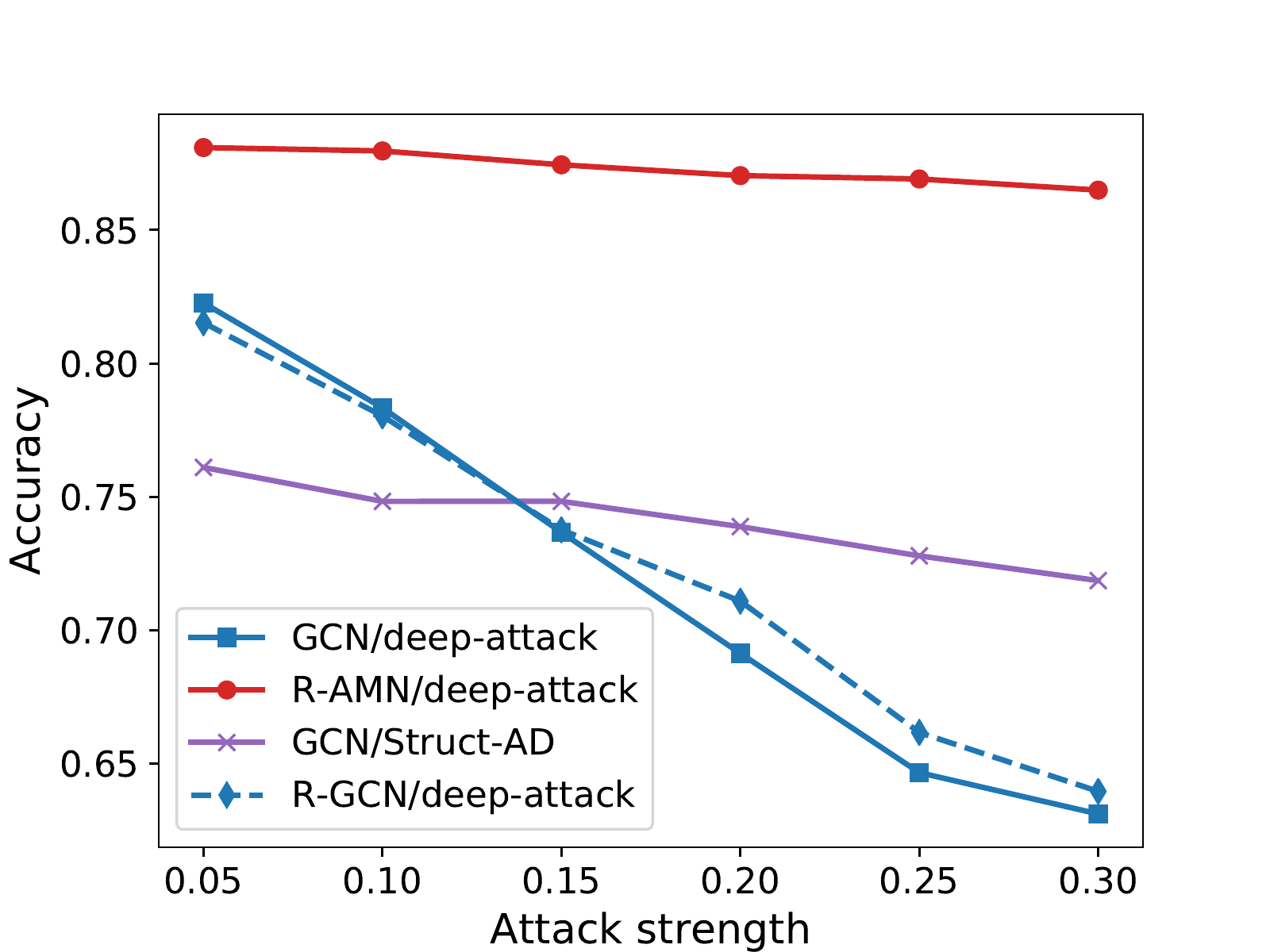}
		\caption{Cora-H}
		\label{fig-deep-attack-Cora-H}
	\end{subfigure}	
	\hfill
	\begin{subfigure}[b]{0.24\textwidth}
		\centering
		\includegraphics[width=\textwidth, height=2cm]{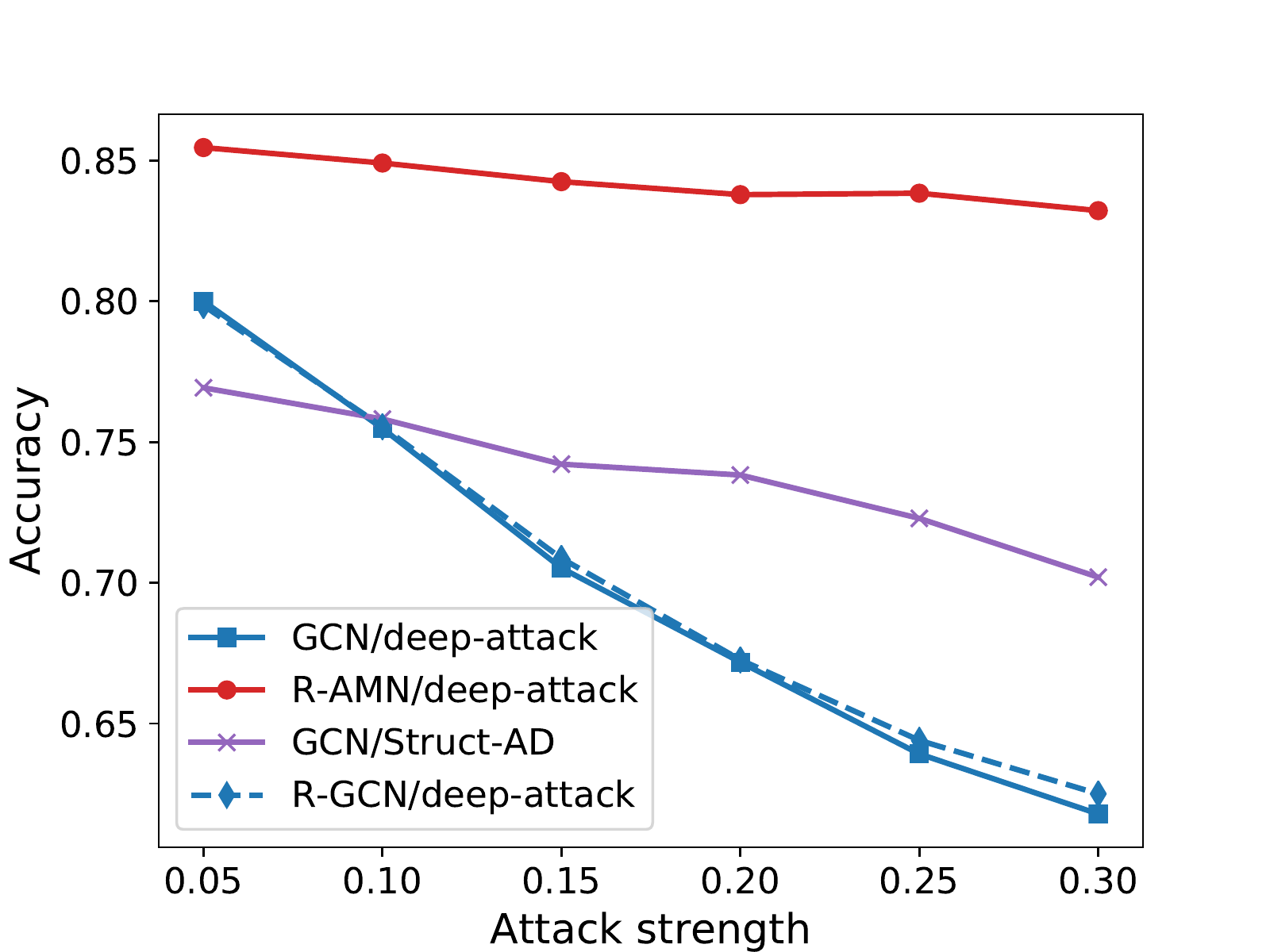}
		\caption{Cora-L}
		\label{fig-deep-attack-Cora-L}
	\end{subfigure}
	\hfill
	\begin{subfigure}[b]{0.24\textwidth}
		\centering
		\includegraphics[width=\textwidth, height=2cm]{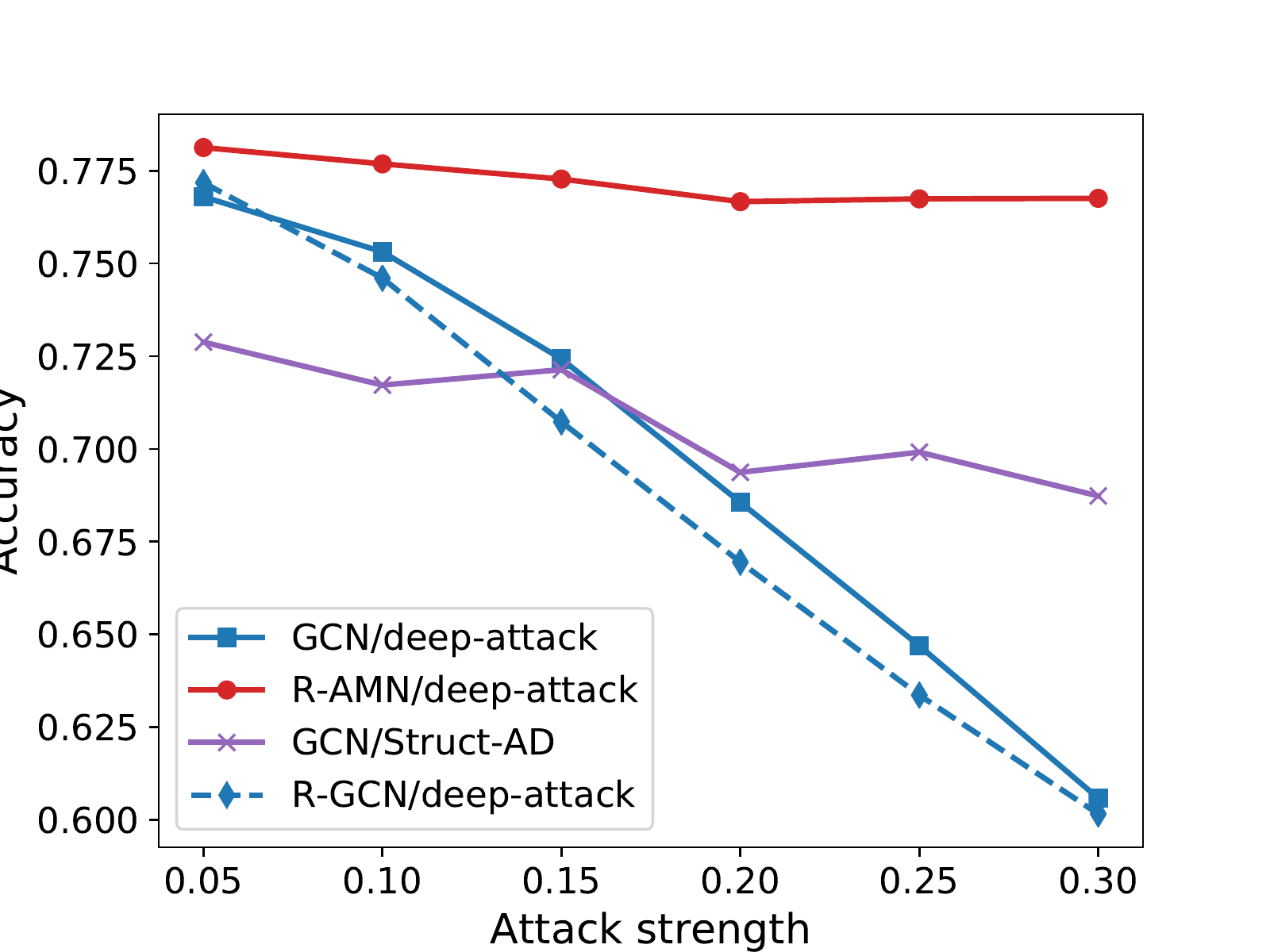}
		\caption{CiteSeer-H}
		\label{fig-deep-attack-Cite-H}
	\end{subfigure}	
	\hfill
	\begin{subfigure}[b]{0.24\textwidth}
		\centering
		\includegraphics[width=\textwidth, height=2cm]{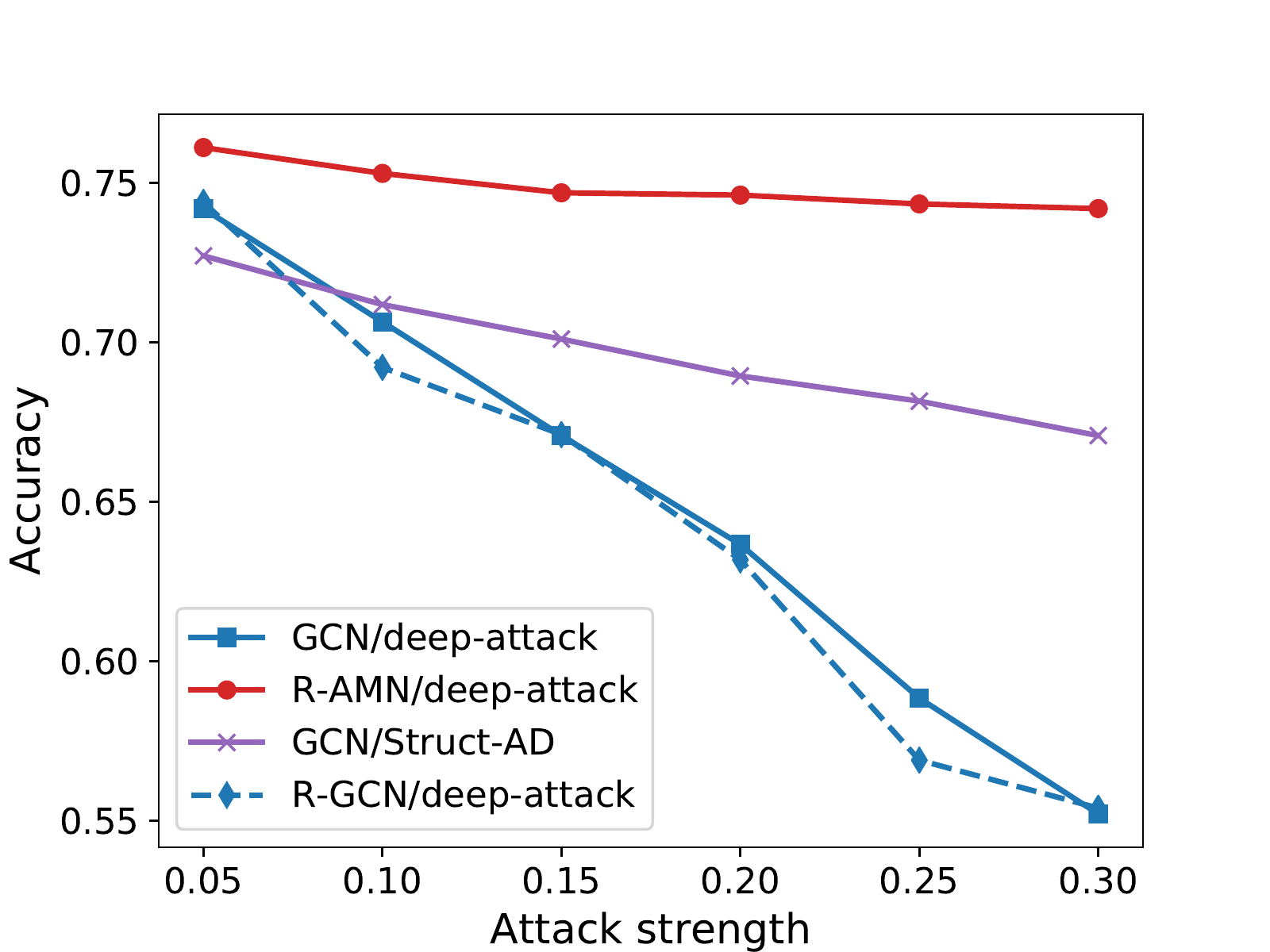}
		\caption{CiteSeer-L}
		\label{fig-deep-attack-Cite-L}
	\end{subfigure}
	\caption{R-AMN and R-GCN under deep attack; GCN under deep attack and Struct-AD attack.} 
	\label{fig-deep-results}
\end{figure}

\begin{figure}[!hb]
	\begin{subfigure}[b]{0.24\textwidth}
		\centering
		\includegraphics[width=\textwidth, height=2cm]{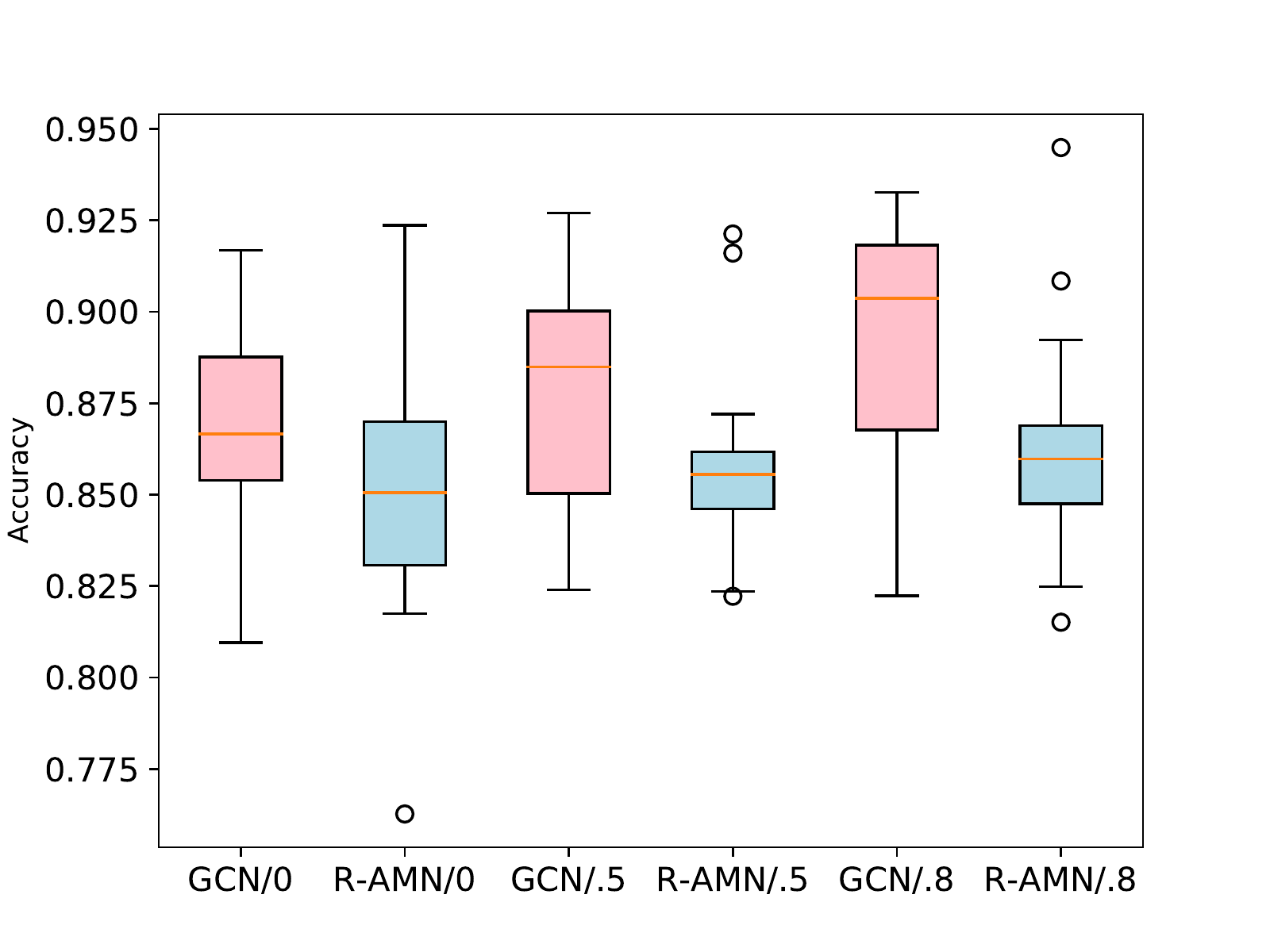}
		\caption{Cora-H}
		\label{fig-Cora-H-compare}
	\end{subfigure}
	\hfill
	\begin{subfigure}[b]{0.24\textwidth}
		\centering
		\includegraphics[width=\textwidth, height=2cm]{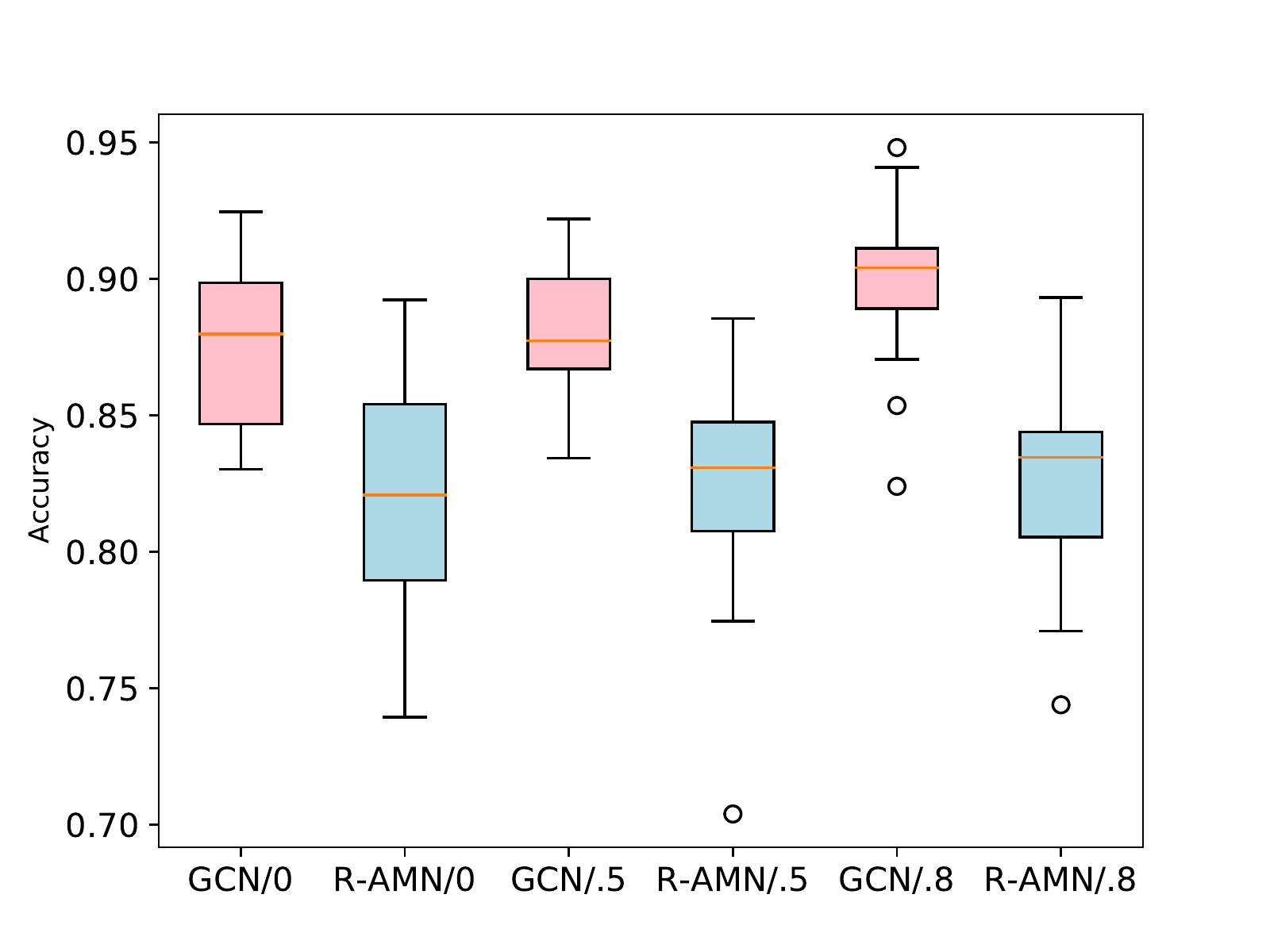}
		\caption{Cora-L}
		\label{fig-Cora-L-compare}
	\end{subfigure}	
	\hfill
	\begin{subfigure}[b]{0.24\textwidth}
		\centering
		\includegraphics[width=\textwidth, height=2cm]{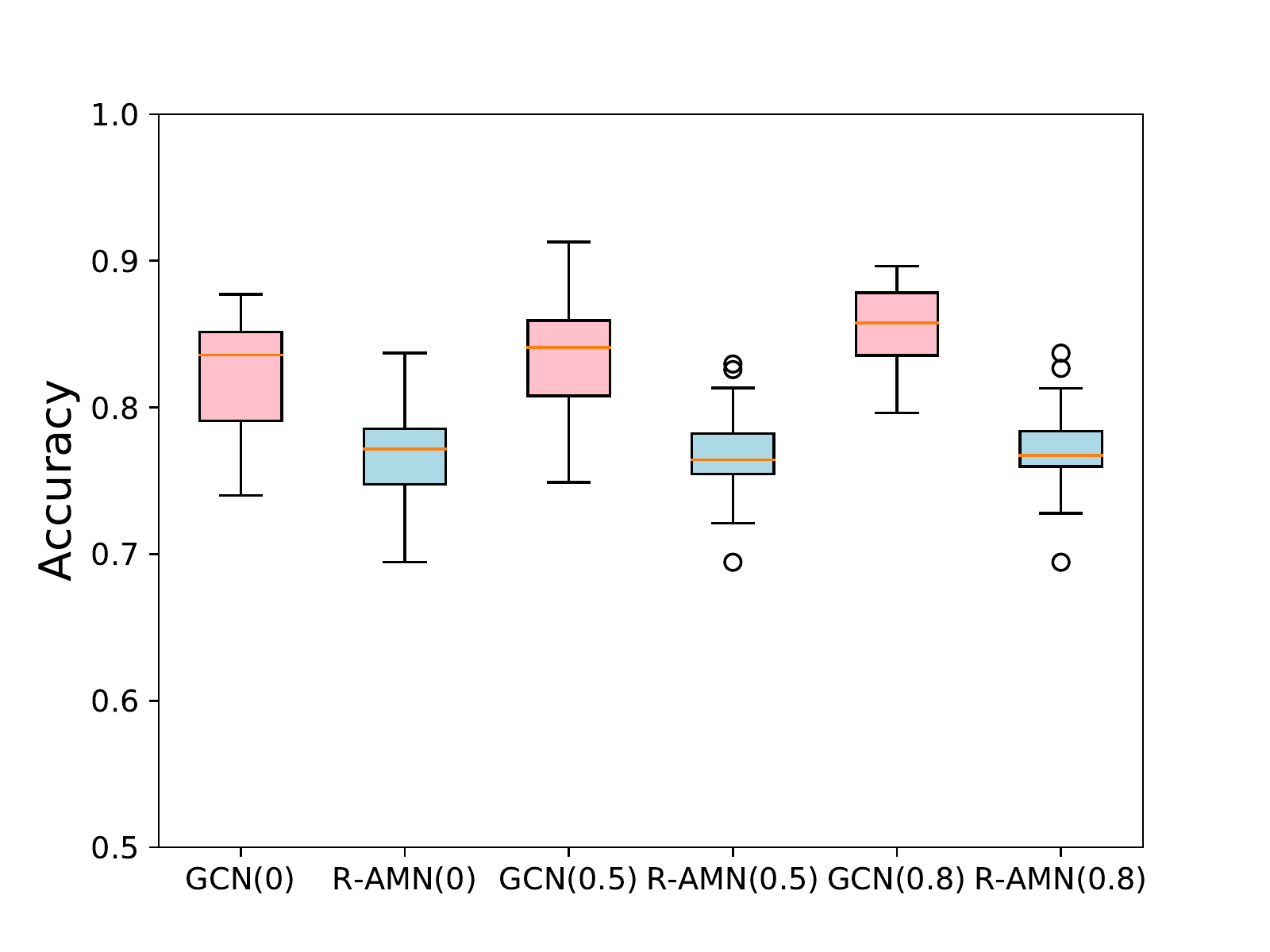}
		\caption{CiteSeer-H}
		\label{fig-CiteSeer-H-compare}
	\end{subfigure}
	\hfill
	\begin{subfigure}[b]{0.24\textwidth}
		\centering
		\includegraphics[width=\textwidth, height=2cm]{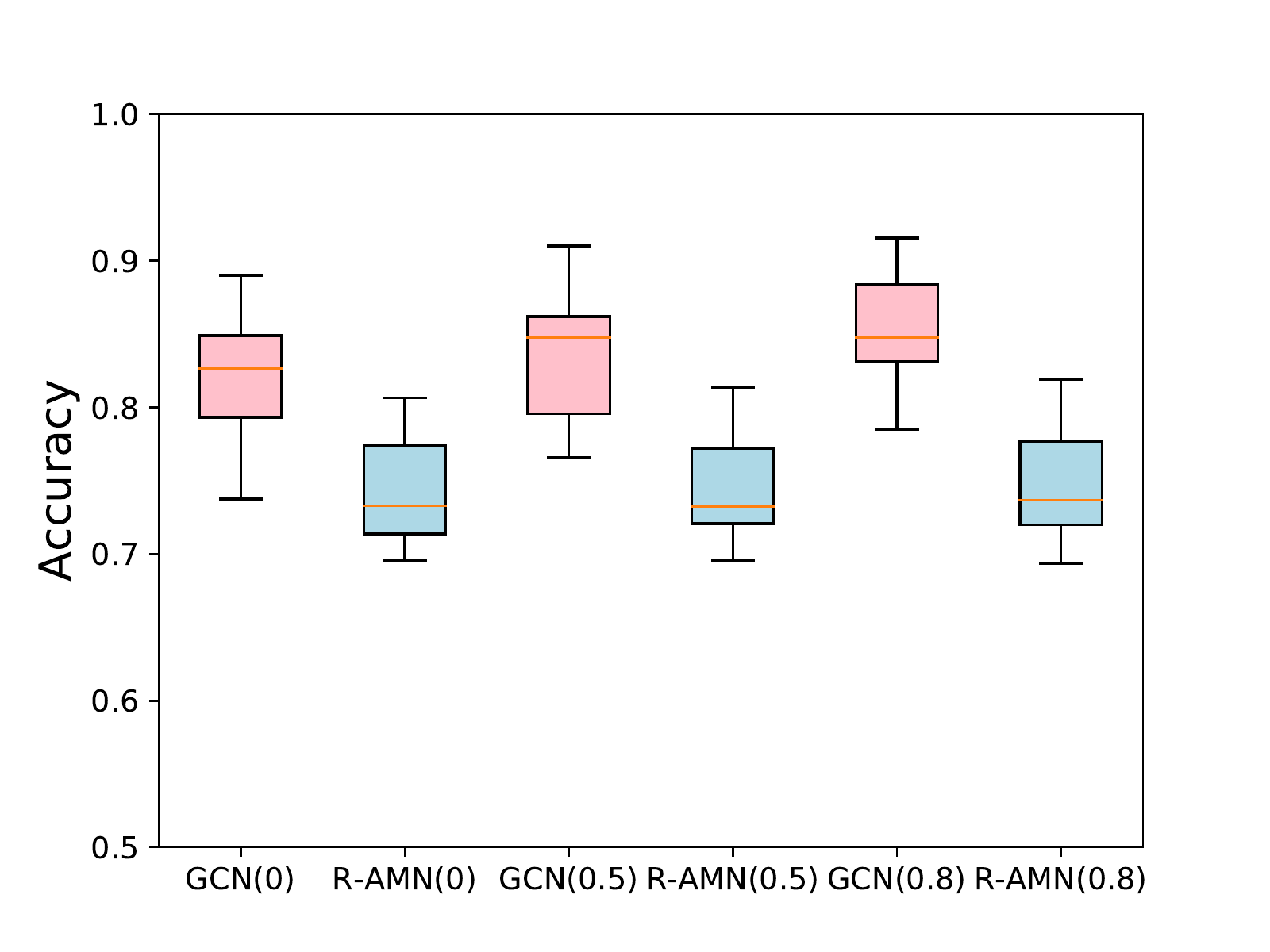}
		\caption{CiteSeer-L}
		\label{fig-CiteSeer-L-compare}
	\end{subfigure}	
	\caption{R-AMN and GCN on non-adversary data as graphs
		are purified, e.g. R-AMN(0.5) stands for R-AMN when
		noisy edges are deleted with probability 0.5.} 
	\label{fig-deep-compare}
\end{figure}
\end{document}